\documentclass[preprint,12pt,authoryear]{elsarticle}

\usepackage{vruler}  

\usepackage[utf8]{inputenc}  
\usepackage{amssymb}
\usepackage{amsthm}
\usepackage{amsmath}
\usepackage{amssymb}
\usepackage{dsfont}
\usepackage{caption} 
\usepackage{float}

\newtheorem{theorem}{Theorem}
\newtheorem{corollary}{Corollary}
\newtheorem{proposition}{Proposition}
\newtheorem{lemma}{Lemma}
\newtheorem{remark}{Remark}
\newtheorem{algorithm}{Algorithm}
\newtheorem{assumption}{Assumption}

\numberwithin{equation}{section}
\newcommand{\ie}{\emph{i.e.}{}}
\newcommand{\eg}{\emph{e.g.}{}}
\newcommand\iid{\ensuremath{\mathit{i.i.d.}}\ }
\newcommand{\rv}{\emph{r.v.}{}}
\newcommand{\wrt}{\emph{w.r.t.}{}}

\usepackage{color}

\newcommand{\ninf}[1]{\| {#1}\|_{\infty}}
\newcommand{\ud}{\mathrm{d}}
\newcommand{\dd}{\{1,\ldots,d\}}

\DeclareMathOperator{\argmin}{argmin}

\def\mb{\mathbf}
\def\point{\,\cdot\,}
\def\P{\mathbb{P}}
\def\PP{\P}
\def\rset{\mathbb{R}}

\def\cone{\mathcal{C}}
\newcommand{\hatmass}{\widehat{\mathcal{M}}}

\journal{Journal of Multivariate Analysis}

\begin{document}

\begin{frontmatter}



\title{Sparse Representation of Multivariate Extremes with Applications to Anomaly Detection}

\author{Nicolas Goix\corref{cor1}}
\cortext[cor1]{Corresponging author.}
\ead{nicolas.goix@telecom-paristech.fr}

\author{Anne Sabourin\corref{cor2}}

\author{Stéphan Clémençon}

\address{LTCI, CNRS, Télécom ParisTech, Université Paris-Saclay\\
46 Rue Barrault, 75013, Paris, France}


\author{}

\address{}

\begin{abstract}
Capturing the dependence structure of multivariate extreme events is
a major concern in many fields involving the management of risks
stemming from multiple sources, \emph{e.g.}~portfolio monitoring, insurance, environmental risk management and anomaly detection.
One convenient (nonparametric) characterization of  extreme dependence in the
framework of multivariate Extreme Value Theory (EVT) is the \textit{angular
  measure}, which provides direct information about the probable
'directions' of extremes, that is, the relative contribution of each
feature/coordinate of the `largest' observations. Modeling the
angular measure in high dimensional problems is a major challenge for
the multivariate analysis of rare events.
The present paper proposes a novel methodology aiming at 
exhibiting a sparsity pattern within the dependence structure of extremes. 
This is achieved by estimating the amount of mass spread by the angular measure on
representative sets of directions, corresponding to  specific sub-cones of $\mathbb{R}_+^d$.
This dimension reduction technique  paves the way towards scaling up existing multivariate EVT methods.
Beyond a non-asymptotic study providing a theoretical validity
framework for our method, we propose  as a direct application a --first--
Anomaly Detection algorithm based on \textit{multivariate} EVT.  This algorithm builds a sparse `normal profile' of extreme behaviours, to be confronted with new (possibly abnormal) extreme observations. Illustrative experimental results provide strong empirical evidence of the relevance of our approach.
\end{abstract}

\begin{keyword}
Multivariate Extremes\sep Anomaly Detection \sep Dimensionality Reduction\sep VC theory


\end{keyword}

\end{frontmatter}



\section{Introduction}
\label{sec:intro}

\subsection{Context: multivariate extreme values in large dimension}

Extreme Value Theory (EVT in abbreviated form) provides a  theoretical basis for
modeling the tails of probability  distributions. In many applied fields where 
rare events  may have a disastrous impact, such as 
finance, insurance, climate,  environmental risk management, network
monitoring (\cite{finkenstadt2003extreme,smith2003statistics}) or  
anomaly detection (\cite{Clifton2011,Lee2008}), the information carried by extremes is crucial. In a multivariate context, the dependence structure
of the joint tail is of particular interest, as it gives access
\emph{e.g.} to probabilities of a joint excess above high thresholds or
to multivariate quantile regions. Also, the distributional structure of
extremes indicates which components of a multivariate quantity may be
simultaneously large while the others stay small, which is a valuable
piece of information for multi-factor risk assessment or
detection of anomalies among other --not abnormal--  extreme data.

In a multivariate `Peak-Over-Threshold' setting, realizations of a $d$ -dimensional random vector $\mb Y = (Y_1 ,..., Y_d)$ are observed and the goal pursued is to learn the conditional distribution of excesses, $\left[~ \mb Y ~|~ \|\mb Y\| \ge  r ~ \right]$, above some large threshold $ r>0$.
The dependence structure of such excesses is described via the
distribution of the ‘directions’ formed by the most extreme
observations, the  so-called \emph{angular measure}, hereafter denoted by $\Phi$.  The latter
is defined on the positive orthant of the $d-1$ dimensional
hyper-sphere. To wit, for any region $A$  on the unit sphere (a set
of `directions'), after suitable standardization of the data (see
 Section~\ref{sec:framework}), 
$C \Phi(A) \simeq \PP(\|\mb Y\|^{-1} \mb Y \in A ~|~ \|\mb Y\| >r)$, where $C$ is a normalizing constant. 
Some probability mass may be spread on any sub-sphere of dimension $k
< d$, the $k$-faces of an hyper-cube if we use the infinity norm, which
complexifies inference when $d$ is large. To fix ideas, the presence of $\Phi$-mass on a
sub-sphere of the type $\{\max_{1\leq i\leq k} x_i = 1 ~;~   x_i >0 \;(i\le k) ~;~  x_{k+1} = \ldots = x_d = 0\}$ indicates that the components $Y_1,\ldots,Y_k$ may
simultaneously be large, while the others are small.
An extensive exposition of this multivariate extreme setting may be found \eg~in \cite{Resnick1987},~\cite{BGTS04}.


Parametric or semi-parametric modeling and estimation of the structure of
multivariate extremes is relatively well documented in the statistical literature, see
\emph{e.g.} \cite{coles1991modeling,fougeres2009models,cooley2010pairwise,sabourinNaveau2012}
and the references therein.
In a non-parametric setting, there is also an abundant literature concerning consistency and asymptotic normality of estimators of functionals characterizing the extreme dependence structure, \eg~extreme value copulas or the \emph{stable tail dependence function} (STDF), see \cite{Segers12Bernoulli}, \cite{Drees98}, \cite{Embrechts2000}, \cite{Einmahl2012}, \cite{dHF06}. 
In many applications, it is nevertheless more convenient to work with the angular measure itself, as the latter gives more direct information on the dependence structure and is able to reflect structural simplifying properties (\eg~sparsity as detailed below) which would not appear in copulas or in the STDF.
However, non-parametric modeling of the angular measure faces major difficulties, stemming from the potentially complexe structure of the latter, especially in a high dimensional setting.
Further, from a theoretical point of view, non-parametric estimation of the angular measure has only been studied in the two dimensional case, in \cite{Einmahl2001} and \cite{Einmahl2009}, in an asymptotic framework.

{Scaling up multivariate EVT} is a major challenge that one
faces when confronted to high-dimensional learning
tasks, since most multivariate extreme value models have been
designed to handle moderate dimensional problems (say, of dimensionality  $d\le 10$). 
For larger dimensions, 
 simplifying modeling choices are needed,
 stipulating \emph{e.g} that only some pre-definite subgroups of components
 may be concomitantly extremes, or, on the contrary, that all of them
 must be (see  \emph{e.g.} 
 \cite{stephenson2009high} or \cite{sabourinNaveau2012}).
This curse of dimensionality can be explained, in the context of
extreme values analysis, by the relative  scarcity of extreme  data,
the  computational
complexity of the estimation procedure and, in the parametric case, by 
the fact that the dimension of the parameter space usually grows with
that of the sample space. This calls for dimensionality reduction devices
adapted to multivariate extreme values.

In a wide range of situations, one may expect the occurrence of two phenomena:

\noindent
\textbf{1-} Only a `small' number of groups of components may be concomitantly extreme, so that only a `small' number of hyper-cubes (those corresponding to these subsets of indexes precisely) have non zero mass (`small' is relative to the total number of groups $2^d$).

\noindent
\textbf{2-} Each of these groups contains a limited number of coordinates (compared to the original dimensionality), so that the corresponding hyper-cubes with non zero mass have small dimension compared to $d$.

\noindent
The main purpose of this paper is to introduce a data-driven
methodology for identifying such faces, so as to reduce the
dimensionality of the problem and thus to learn a sparse 
representation  of extreme behaviors. 
In case hypothesis \textbf{2-} is not fulfilled, such a sparse  `profile' can still be learned, but looses the low dimensional property of its supporting hyper-cubes.

One major issue is that real data generally do not concentrate on sub-spaces of zero Lebesgue measure. This is circumvented by setting to zero any coordinate less than a threshold $\epsilon>0$, so that the corresponding `angle' is assigned to a lower-dimensional face.





The theoretical results stated in this paper build on the work
of 
\cite{COLT15}, where non-asymptotic bounds related to the statistical performance of a non-parametric estimator of the STDF, another functional measure of the dependence structure of extremes,  are established.  
However, even in the case of a sparse angular measure, the support of
the STDF would not be so, since the latter functional is  an
integrated version of the former (see~\eqref{eq:integratePhiLambda},
 Section~\ref{sec:framework}). Also, 
in many applications, it is more convenient to work with 
 the {angular
  measure}. Indeed, it  provides
direct  information about the
probable `directions' of extremes, that is, the relative contribution
of each components of the `largest' observations  (where `large' may be 
understood \emph{e.g.} in the sense of the infinity norm on the input
space). We emphasize again that estimating these `probable relative contributions' is a major concern in many fields
involving  the management of risks from multiple sources.
To the best of our knowledge, non-parametric estimation of the angular
measure has only been treated in the two dimensional case, in
\cite{Einmahl2001} and \cite{Einmahl2009}, in an asymptotic
framework.

\noindent
\textbf{Main contributions.} The present paper extends the non-asymptotic bounds proved in \cite{COLT15} to the angular measure of extremes, restricted to a well-chosen representative class of sets, corresponding to lower-dimensional regions of the space. The objective is to learn a representation of the angular measure, rough enough to
control the variance in high dimension and accurate enough to gain information about the 'probable directions' of extremes. This yields a --first-- non-parametric estimate of the angular measure in any dimension, restricted to a
class of sub-cones, with a non asymptotic bound on the error. 
 The representation thus obtained is exploited to detect anomalies among extremes. 

The proposed algorithm  is based on \textit{dimensionality
  reduction}. 
We believe that our method can also be used as a preprocessing stage, for dimensionality reduction purpose, before proceeding with a parametric or semi-parametric estimation which could benefit from the structural information issued in the first step. Such applications are beyond the scope of this paper and will be the subject of further research.

\subsection{Application to Anomaly Detection}
Anomaly Detection (AD in short, and depending of the application domain, outlier detection, novelty detection, deviation detection, exception mining) generally consists in assuming that the dataset under study contains a \textit{small} number of anomalies, generated by distribution models that  \textit{differ} from that generating the vast majority of the data.
This formulation motivates many statistical AD methods, based on the underlying assumption that anomalies occur in low probability regions of the data generating process. Here and hereafter, the term `normal data' does not refer to Gaussian distributed data, but  to  \emph{not abnormal} ones, \ie~data belonging to the above mentioned majority. 
Classical parametric techniques, like those developed in \cite{Barnett94} or in \cite{Eskin2000}, assume that the normal data are generated by a distribution belonging to some  specific, known in advance parametric model.  
The most popular non-parametric approaches include algorithms based on density (level set) estimation (see \textit{e.g.} \cite{Scholkopf2001},  \cite{Scott2006} or \cite{Breunig99LOF}), on dimensionality reduction (\textit{cf} \cite{Shyu2003}, \cite{Aggarwal2001}) or on decision trees (\cite{Liu2008}).
One may refer to \cite{Hodge2004survey}, \cite{Chandola2009survey}, \cite{Patcha2007survey} and \cite{Markou2003survey} for excellent overviews of current research on Anomaly Detection, ad-hoc techniques being far too numerous to be listed here in an exhaustive manner.
The framework we develop in this paper is non-parametric and lies at
the intersection of support estimation, density estimation and
dimensionality reduction: it consists in learning from training data
the support of a distribution, that can be decomposed into sub-cones,
hopefully of low dimension each and to which some mass is assigned,
according to 
empirical versions of probability measures on 
extreme regions. 

EVT has been intensively used in AD in the one-dimensional
situation, see for instance \cite{Roberts99}, \cite{Roberts2000},
\cite{Clifton2011}, \cite{Clifton2008}, \cite{Lee2008}.
In the multivariate setup, however, there is --to the best of our
knowledge--  no anomaly detection method
relying on \textit{multivariate} EVT. Until now, the multidimensional case has only been  tackled by means of extreme value statistics
based on univariate EVT. The major reason is 
the difficulty to scale up existing multivariate EVT models
with the dimensionality.
In the present paper we bridge the gap between the practice of AD and multivariate EVT by proposing a method which is
able to learn a sparse `normal profile' of multivariate extremes and,
as such, may be implemented to improve the accuracy of  any usual AD
algorithm. 
Experimental results show that this method significantly
improves the performance in extreme regions, as  
the risk
is taken not to uniformly predict as abnormal the most extremal observations, but to learn their dependence structure.
These improvements may typically be useful in applications where the cost of false positive errors (\ie~false alarms) is very high (\eg~predictive maintenance in aeronautics).

\bigskip
The structure of the paper is as follows.  The whys
and wherefores of multivariate EVT are explained in the following
Section~\ref{sec:framework}. A non-parametric estimator of the
subfaces' mass is introduced in Section~\ref{sec:estimation}, the
accuracy of which is
investigated by establishing 
finite sample error bounds relying on  {\sc VC} inequalities
tailored to low probability regions.
An application to Anomaly Detection is proposed in Section~\ref{sec:appliAD}, where some background on AD is provided, followed by a novel AD  
 algorithm which relies on the above mentioned non-parametric
 estimator. 
Experiments on both simulated and real data are performed in Section~\ref{sec:experiments}. Technical details are deferred to the Appendix section.

\section{Multivariate EVT Framework and Problem Statement}
\label{sec:framework}
%
%

Extreme Value Theory (\textsc{EVT}) develops models for learning the
unusual rather than the usual, in order to provide a reasonable
assessment of the probability of occurrence of rare events. Such models are widely used in fields
involving risk management such as Finance, Insurance, Operation Research, Telecommunication
or Environmental Sciences for instance. For clarity, we start off with recalling some key notions pertaining to (multivariate) \textsc{EVT}, that shall be involved in the formulation of the problem next stated and in its subsequent analysis. 

\subsection{Notations}
Throughout the paper, bold symbols refer to multivariate quantities, and for $m \in \mathbb{R}\cup \{\infty\}$, $\mb m$ denotes the vector $(m,\ldots,m)$.
Also, comparison operators  between two vectors (or between a vector and a real number) are
understood component-wise, \ie~ `$\mb x \le \mb z$' means `$x_j \le z_j$
for all $1\le j\le d$' and  for any real number $T$,  `$\mb x\le T$' means `$x_j \le T$ for all $1\le j\le d$'. 
We denote by $\lfloor u \rfloor$ the integer part of any real number $u$, by $u_+=\max(0,\; u)$ its positive part and by $\delta_{\mb a}$ the Dirac mass at any point $\mb a\in \mathbb{R}^d$.  
For uni-dimensional random variables $Y_1,\ldots,Y_n$, $Y_{(1)} \le \ldots\le Y_{(n)}$ denote their order statistics.

\subsection{Background on (multivariate) Extreme Value Theory}

In the univariate case, \textsc{EVT} essentially consists in modeling
the distribution of the maxima ({\it resp.} the upper tail of the \rv~under study) as a {\it generalized
extreme value distribution}, namely an element of the Gumbel, Fr\'echet
or Weibull parametric families ({\it resp.} by a generalized Pareto distribution).
It plays a crucial role in risk monitoring: 
consider the $(1-p)^{th}$
quantile of the distribution  $F$ of a r.v. $X$,  for a given exceedance probability $p$, that is
$x_p = \inf\{x \in \mathbb{R},~ \mathbb{P}(X > x) \le p\}$. For
moderate values of $p$, a natural empirical estimate is  $x_{p,n} = \inf\{x \in
\mathbb{R},~ 1/n \sum_{i=1}^n \mathds{1}_{\{X_i > x\}}\le p\}$.
However,  if
$p$ is very small, the finite  sample $X_1,\; 
\ldots, X_n$  carries insufficient information and the empirical quantile $x_{p,n}$ becomes 
unreliable. 
That is where \textsc{EVT} comes into play  by providing
parametric estimates of large
quantiles: 
whereas statistical inference often involves sample means and the
Central Limit
Theorem, 
\textsc{EVT} handles phenomena whose behavior is 
not ruled by an `averaging effect'. The focus is on the sample maximum
rather than the mean. The primal assumption is the existence of two
sequences $\{a_n, n \ge 1\}$ and $\{b_n, n \ge 1\}$, the $a_n$'s being
positive, and a non-degenerate distribution function $G$ such that
\begin{equation}
\label{intro:assumption1}
\lim_{n \to \infty} n ~\mathbb{P}\left( \frac{X - b_n}{a_n} ~\ge~ x \right) = -\log G(x)
\end{equation}
for all continuity points $x \in \mathbb{R}$ of $G$.
If this assumption is fulfilled -- it is the case for most textbook
distributions -- then $F$ is said to lie in the \textit{domain of
  attraction} of $G$: $F \in DA(G)$. The tail behavior of $F$
is then essentially characterized by $G$, which is proved to be -- up
to  re-scaling -- of the type $G(x) = \exp(-(1 + \gamma
x)^{-1/\gamma})$ for $1 + \gamma x > 0$, $\gamma \in \mathbb{R}$,
setting by convention $(1 + \gamma x)^{-1/\gamma} = e^{-x}$ for
$\gamma = 0$. The sign of $\gamma$ controls the shape of the tail and
various estimators of the re-scaling sequence and of the shape index $\gamma$ as well have
been studied in great detail, see \emph{e.g.}  \cite{DEd1989},
\cite{ELL2009}, 
\cite{Hill1975}, \cite{Smith1987}, \cite{BVT1996}. 
\medskip


\noindent \textbf{Extensions to the multivariate setting} are well understood
from a probabilistic point of view, but far from obvious from a
statistical perspective. Indeed, the tail dependence structure, ruling the possible simultaneous occurrence of large observations in several directions, has no finite-dimensional parametrization.

The analogue of (\ref{intro:assumption1}) for a $d$-dimensional \rv $\mb X = (X^1,\; \ldots, \; X^d)$ with distribution $\mb F(\mb x):=\mathbb{P}(X_1 \le x_1, \ldots, X_d \le x_d)$, namely $\mb F \in \textbf{DA} (\mb G)$ stipulates the existence of two sequences $\{\mb a_n, n \ge 1\}$ and $\{\mb b_n, n \ge 1\}$ in $\mathbb{R}^d$, the $\mb a_n$'s being positive,
and a non-degenerate distribution function $\mb G$ such that
\begin{equation}
\label{intro:assumption2}
\lim_{n \to \infty} n ~\mathbb{P}\left( \frac{X^1 - b_n^1}{a_n^1} ~\ge~ x_1 \text{~or~} \ldots \text{~or~} \frac{X^d - b_n^d}{a_n^d} ~\ge~ x_d \right) = -\log \mb G(\mathbf{x})
\end{equation}
for all continuity points $\mathbf{x} \in \mathbb{R}^d$ of $\mb G$. This clearly implies
that the margins $G_1(x_1),\ldots,G_d(x_d)$ are univariate extreme
value distributions, namely of the type $G_j(x) = \exp(-(1 + \gamma_j
x)^{-1/\gamma_j})$. Also, denoting by $F_1,\; \ldots,\; F_d$ the
marginal
distributions of $\mb F$, Assumption~\eqref{intro:assumption2} implies marginal convergence: $F_i \in DA(G_i)$ for $i=1,\; \ldots,\; n$.
 To understand 
the structure of the limit $\mb G$ and dispose of the
 unknown sequences $(\mb a_n, \mb b_n)$ (which are entirely determined by the
 marginal distributions $F_j$'s), 
 it is convenient to
 work with marginally standardized variables, that is, to separate the margins from the dependence structure in the description of the joint distribution of $\mb X$. Consider the standardized variables 
 $V^j =1/(1-F_j(X^j))$ and $\mathbf{V}=(V^1,\; \ldots,\; V^d)$.  In
 fact (see Proposition 5.10 in \cite{Resnick1987}), Assumption~\eqref{intro:assumption2} is
 equivalent to marginal convergences $F_j \in DA(G_j)$ as in (\ref{intro:assumption1}),  
 together with standard  multivariate regular variation of $\mathbf{V}$'s
 distribution,  which means existence of a limit measure $\mu$  on $ [0,\infty]^d\setminus\{\mb 0\}$ such that 
\begin{equation}
\label{intro:regvar}
 n~ \mathbb{P}\left( \frac{V^1 }{n} ~\ge~ v_1 \text{~or~} \cdots
   \text{~or~} \frac{V^d }{n} ~\ge~ v_d \right) \xrightarrow[n\to\infty]{}\mu \left([\mb 0,\mb v]^c\right),
\end{equation}
where $[\mb 0,\mathbf{v}]:=[0,\; v_1]\times \cdots \times
[0,\; v_d]$. Thus, the variable $\mb V$ satisfies
(\ref{intro:assumption2}) with $\mb a_n = \mb n = (n,\; \ldots,\; n)$, $\mb b_n =\mb 0 =(0,\; \ldots,\; 0)$. The dependence structure of the limit $\mb G$ in (\ref{intro:assumption2})
can be expressed by means of the so-termed \textit{exponent measure} $\mu$: 
\begin{equation}
- \log \mb G(\mathbf{x})= \mu\left( \left[ \mb 0, \left(\frac{-1}{\log G_1(x_1)}, \dots ,\frac{-1}{\log G_d(x_d)}\right)\right]^c\right). \nonumber
\end{equation}
The latter  is finite on
sets bounded away from $\mb 0$ and has the
homogeneity property : $\mu(t\point) =
t^{-1}\mu(\point)$. Observe in addition that, due to the standardization chosen (with
`nearly' Pareto margins), the support of $\mu$ is included in $[\mb 0,\; \mathbf{1}]^c$. 
 To wit, the measure $\mu$ should be viewed, up to a a normalizing factor, as
the asymptotic distribution of $\mb V$ in extreme regions. For any borelian subset $A$ bounded away from $\mb 0$ on which $\mu$ is continuous, we have 
\begin{equation}
\label{eq:regularVariation}
t~ \mathbb{P}\left( \mb V \in t A\right) \xrightarrow[t\to\infty]{}\mu(A).     
\end{equation}
Using the homogeneity property $\mu(t\point) =
t^{-1}\mu(\point)$, one may show
that $\mu$  can be decomposed into a  radial component and an angular component
$\Phi$, which are independent from each other (see \emph{e.g.} \cite{dR1977}).
Indeed, for all $\mb v = (v_1,...,v_d) \in \mathbb{R}^d$, set
\begin{equation}\label{eq:pseudoPolar_change}
  \left\{ \begin{aligned}
R(\mb v)&:= \|\mb v\|_\infty ~=~ \max_{i=1}^d v_i, \\
\Theta (\mb v) &:= \left( \frac{v_1}{R(\mb v)},..., \frac{v_d}{R(\mb v)} \right)
\in S_\infty^{d-1},     
  \end{aligned}\right.
\end{equation}
where $S_\infty^{d-1}$ is the positive orthant of  the unit sphere in $\mathbb{R}^d$ for the infinity norm.
Define the \emph{ spectral measure} (also called \emph{angular measure}) by $\Phi(B)= \mu (\{\mb v~:~R(\mb v)>1 ,
\Theta(\mb v) \in B \})$. Then, for every $B
\subset S_\infty^{d-1}$, 
\begin{equation}
\label{mu-phi}
\mu\{\mb v~:~R(\mb v)>z, \Theta(\mb v) \in B \} = z^{-1} \Phi (B)~. 
\end{equation}
In a nutshell,  there
is a one-to-one correspondence between the exponent measure $\mu$ and the angular measure
$\Phi$, both of them can be used to characterize the asymptotic tail
dependence of the distribution $\mb F$ (as
soon as the  margins $F_j$ are known), since   
\begin{equation}\label{eq:integratePhiLambda}
  \mu \big( [\mb 0,\mathbf{x}^{-1}]^c \big) =  \int_{\boldsymbol{\theta} \in S_{\infty}^{d-1}}   \max_j{\boldsymbol{\theta}_j x_j} \;\ud \Phi(\boldsymbol{\theta}),
\end{equation}
this equality being obtained from the change of variable~\eqref{eq:pseudoPolar_change} , see \emph{e.g.} Proposition 5.11 in \cite{Resnick1987}. 
Recall that here and beyond, operators on vectors are understood component-wise, so that $\mb x^{-1}=(x_1^{-1},\ldots,x_d^{_1})$.
The angular measure can be seen as the asymptotic conditional distribution of the
`angle' $\Theta$ given that the radius $R$ is large, up to the
normalizing constant $\Phi(S_\infty^{d-1})$. Indeed, dropping
the dependence on $\mb V$ for convenience, we have for any
\emph{continuity set} $A$ of $\Phi$, 
\begin{equation}
  \label{eq:limitConditAngle}
\begin{aligned}
  \PP(\Theta \in A ~|~R>r ) &= 
\frac{r  \PP(\Theta \in A , R>r ) }{r\PP(R>r)} 
& \xrightarrow[r\to \infty]{} \frac{\Phi(A)}{\Phi(S_\infty^{d-1})} .
\end{aligned}  
\end{equation}
The choice of the marginal standardization is somewhat arbitrary and alternative standardizations  lead
to different limits. Another common choice consists in considering `nearly
uniform' 
variables (namely, uniform variables when the margins are continuous): defining $\mathbf{U}$ by $U^j =1-F_j(X^j)$ for
$j\in\{1,\ldots,d\}$,   
Condition (\ref{intro:regvar}) is equivalent to each of the  following conditions:
\begin{itemize}
\item $\mathbf{U}$ has  `inverse multivariate regular variation' 
  with limit measure $\Lambda(\point)$ $:=\mu((\point)^{-1})$, namely,
  for every measurable set $A$ bounded away from $+\boldsymbol{\infty}$ which is a
  continuity set of $\Lambda$,
\begin{equation}
\label{reg_var_U}
t~ \mathbb{P}\left( \mb U \in t^{-1} A\right)
\xrightarrow[t\to\infty]{} \Lambda(A) = \mu(A^{-1}), 
\end{equation}
where $A^{-1} = \{\mb u \in \rset^{d}_+ ~:~(u_1^{-1},\ldots,u_d^{-1})
\in A\}$. The limit measure $\Lambda$ is finite on sets bounded away from $\{+\boldsymbol{\infty}\}$. 
\item The \textit{stable tail dependence function} (STDF) defined for $\mb x\in[\mb 0,\boldsymbol{\infty}], \mb x\neq\boldsymbol{\infty}$ by 
\begin{equation}
\label{stdf1}
l(\mb x) = \lim_{t \to 0} t^{-1} \mathbb{P} \left( U^1 \le t\, x_1 ~\text{or}~ \ldots ~\text{or}~ U^d \le t\,x_d  \right)
 = \mu\left([\mb 0, \mb{x}^{-1}]^c\right) 
\end{equation}
exists. 
\end{itemize}

\subsection{Statement of the Statistical Problem}\label{sec:decomposMu}

The focus of this work  is on 
the dependence structure in extreme regions of a 
random vector $\mb X$ in a multivariate domain of attraction (see
(\ref{intro:assumption1})). This asymptotic dependence   
 is fully described by the exponent measure $\mu$, or
equivalently by the spectral measure $\Phi$. The goal 
 of this paper is to  infer a  meaningful  (possibly sparse) summary of the latter.   
 As shall be seen below,
since the support of $\mu$ can be naturally partitioned in a specific
and interpretable manner, this boils down to accurately recovering the
mass spread on each element of the partition.  In order to formulate
this approach rigorously, additional 
definitions are required.
\medskip

 \noindent{\bf Truncated cones}. For any non empty subset of features $\alpha\subset\{1,\; \ldots,\; d \}$, consider the truncated cone (see Fig.~\ref{fig:3Dcones})
 \begin{equation}
 \label{cone}
 \mathcal{C}_\alpha = \{\mb v \ge 0,~\|\mb v\|_\infty \ge 1,~ v_j > 0 ~\text{ for } j \in \alpha,~ v_j = 0 ~\text{ for } j \notin \alpha \}.
 \end{equation}
The corresponding subset of the sphere is 
\begin{equation}
\Omega_{\alpha}  = \{\mb x \in S_{\infty}^{d-1} :  x_i > 0 \text{ for } i\in\alpha~,~  x_i = 0 \text{ for } i\notin \alpha   \} 
 = S_{\infty}^{d-1}\cap {\mathcal{C}}_\alpha, \nonumber
\end{equation}
and we clearly have $\mu(\mathcal{C}_\alpha) =  \Phi(\Omega_\alpha)$ for any $\emptyset\neq \alpha \subset\{1,\; \ldots,\; d \}$.
The collection $\{\mathcal{C}_\alpha:\; \emptyset \neq
\alpha\subset \{1,\; \ldots,\; d \}\}$ forming a partition of the
truncated positive orthant $\mathbb{R}_+^{d}\setminus[\mb 0,\mb 1]$, one may naturally decompose the exponent measure as 
\begin{equation}\label{eq:decomp1}
 \mu = \sum_{\emptyset \neq \alpha\subset\{1,\ldots ,d\}}
\mu_\alpha,
\end{equation} 
where each component $\mu_\alpha$ is concentrated on the
untruncated cone corresponding to ${\cal C_\alpha}$.
Similarly, 
 the $\Omega_\alpha $'s forming  a partition of
$S_\infty^{d-1}$, we have 
\begin{equation}
 \Phi ~=~ \sum_{\emptyset \neq \alpha\subset\{1,\ldots ,d\}} \Phi_\alpha ~, \nonumber
\end{equation}  
where $\Phi_\alpha$ denotes the restriction of $\Phi$ to 
${\Omega}_\alpha$ for all $\emptyset\neq \alpha \subset\{1,\; \ldots,\; d \}$.
The fact that mass is spread on   $\cone_\alpha$ indicates that conditioned upon
the event `$R(\mb V)$ is large' (\ie~an excess of a large radial threshold),
the components $V^j (j\in\alpha)$ may be simultaneously large while
the other  $V^j$'s  $(j\notin\alpha)$ are small, with positive
probability.
Each index subset $\alpha$ thus defines a specific direction in the tail region. 

 However this interpretation should be handled with care,
since  for $\alpha\neq\{ 1,\ldots,d\}$,  if $\mu(\cone_\alpha)>0$,
then $\cone_\alpha$  is not a continuity set of $\mu$
(it has empty interior), nor $\Omega_\alpha$ is a continuity set
of $\Phi$. Thus, the quantity $t \PP(\mb V \in t \cone_\alpha)$ does not necessarily converge to
$\mu(\cone_\alpha)$ as $t\rightarrow +\infty$. 
Actually, if $\mb F$ is continuous, we have $\PP(\mb V \in t \cone_\alpha) =0$
for any $t>0$. However, consider for $\epsilon \ge 0$ the {\it $\epsilon$-thickened rectangles }
\begin{equation}
 \label{eq:epsilon_Rectangle} 
 R_\alpha^\epsilon~=\{\mb v \ge 0,~\|\mb v\|_\infty \ge 1,~ v_j > \epsilon  ~\text{ for } j \in \alpha,
~v_j \le \epsilon ~\text{ for } j \notin \alpha
 \} ,
\end{equation}
Since the boundaries of the sets $R_\alpha^\epsilon$ are disjoint, only a countable number of them may be discontinuity sets of $\mu$. Hence, the threshold $\epsilon$ may be chosen arbitrarily small  in such a way that
$R_\alpha^\epsilon$ is a continuity set of $\mu$. 
The result stated below
shows that nonzero mass on $\cone_\alpha$ is the same as
nonzero mass on $R_\alpha^\epsilon$  for $\epsilon$ arbitrarily small.

\noindent
\begin{minipage}{0.5\linewidth}
\centering
\includegraphics[scale=0.2]{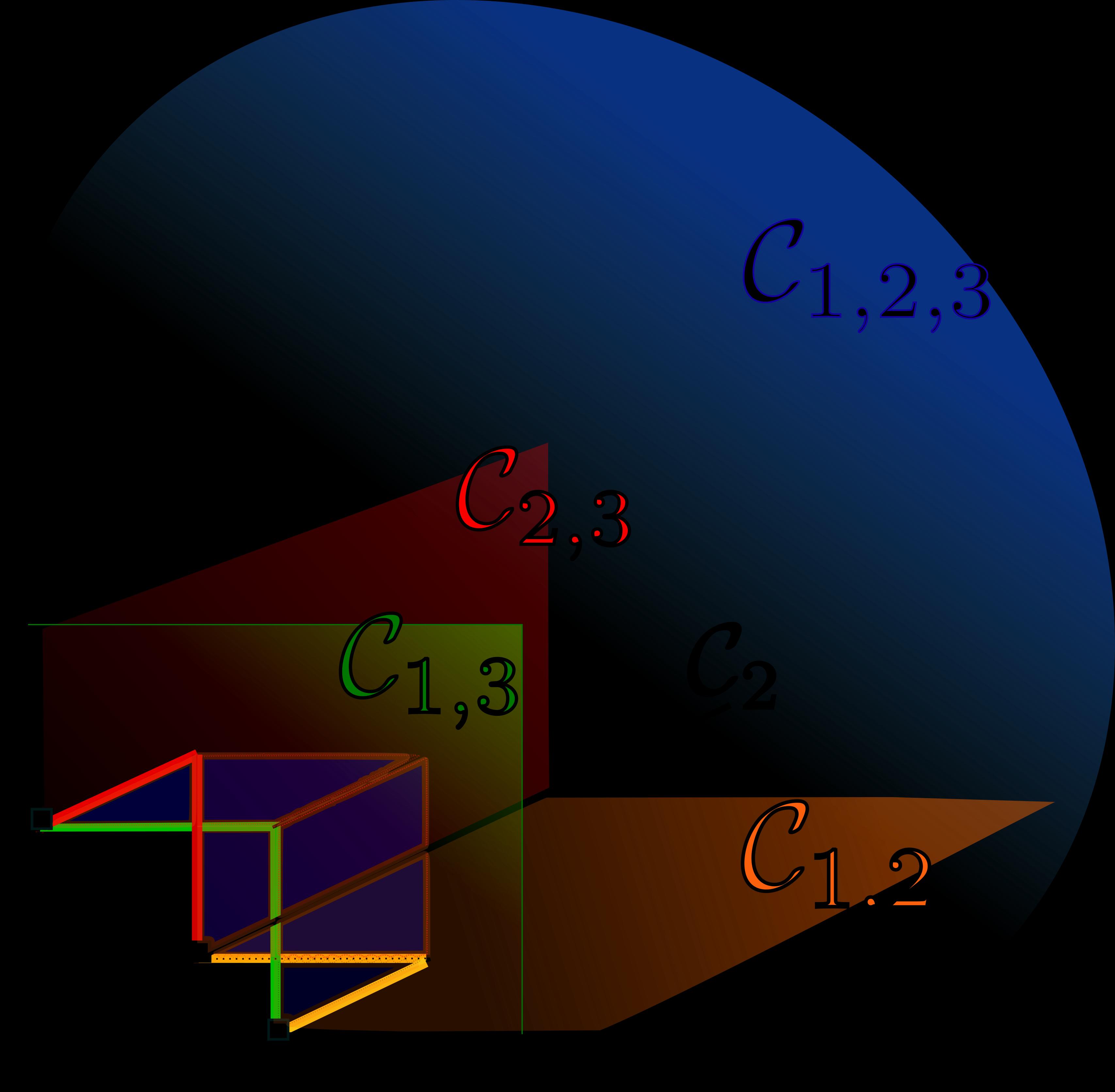}
\captionof{figure}{Truncated cones in 3D}
\label{fig:3Dcones}
\end{minipage}\hfill
\begin{minipage}{0.5\linewidth}
\centering
\includegraphics[width=0.71\linewidth]{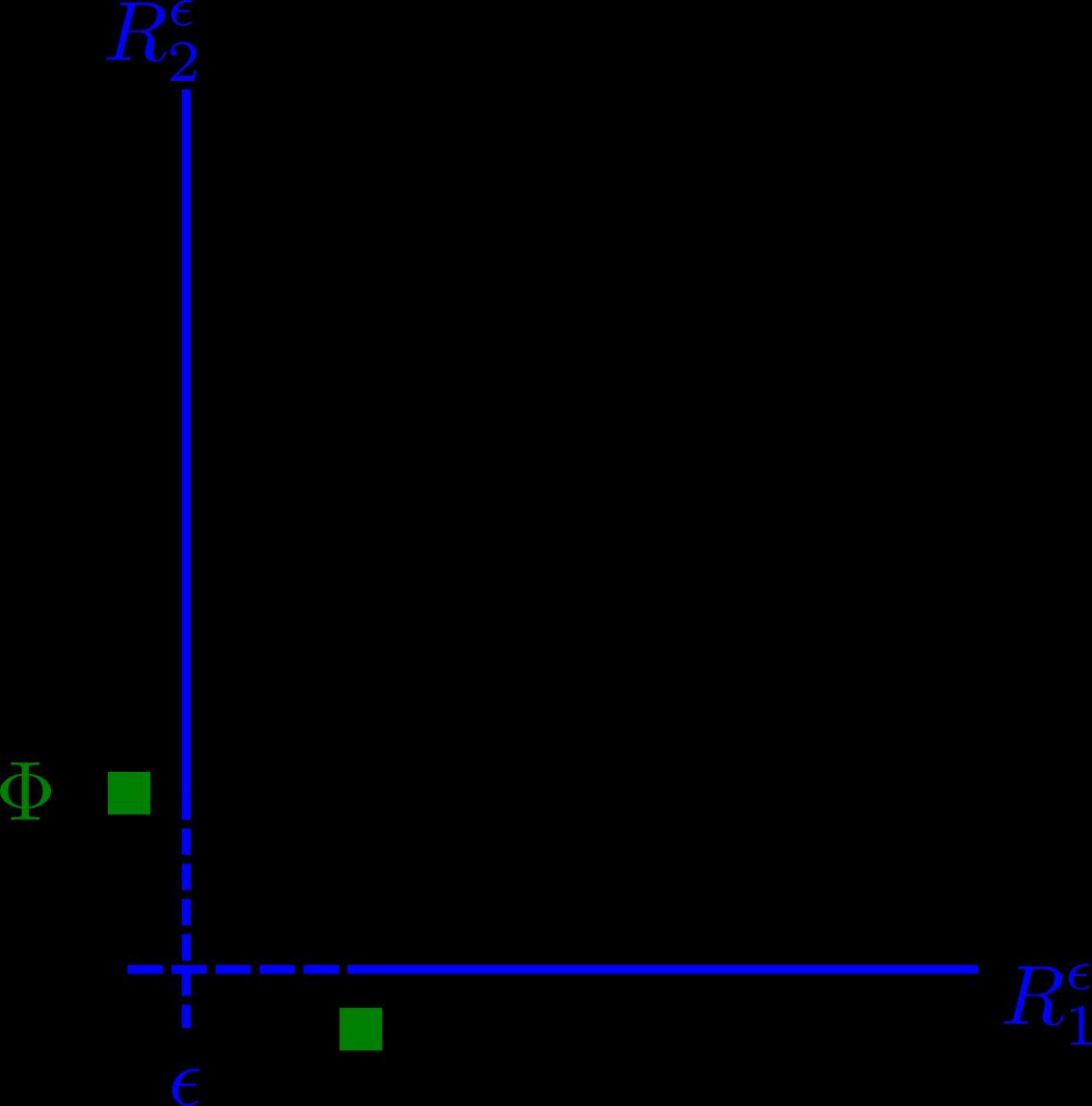}
\captionof{figure}{Truncated $\epsilon$-rectangles in 2D}
\label{2Dcones}
\end{minipage}

\begin{lemma}\label{lem:limit_muCalphaEps}
For any non empty index subset $\emptyset \neq \alpha\subset\{1,\ldots,d\}$, the exponent measure of
$\cone_\alpha$ is
\[
\mu(\cone_\alpha) = \lim_{\epsilon\to 0} \mu(R_\alpha^\epsilon).
\]
\end{lemma}
\begin{proof}
First consider the case $\alpha=\dd$. Then $R_\alpha^\epsilon$'s forms an increasing sequence of sets as $\epsilon$ decreases and $\mathcal{C}_\alpha = R_\alpha^0 = \cup_{\epsilon>0, \epsilon \in \mathbb{Q}}~R_\alpha^\epsilon$. The result follows from the `continuity from below' property of the measure $\mu$. 
Now, for $\epsilon\ge 0$ and $\alpha\subsetneq\{1,\; \ldots,\; d\}$, consider the sets
\begin{equation}
\begin{aligned}
O_\alpha^\epsilon &  =\{ \mb x \in\rset_+^d~: \forall j \in \alpha:x_j > \epsilon \},  \nonumber \\
N_\alpha^\epsilon &  =\{\mb x \in\rset_+^d~: \forall j \in \alpha:   x_j > \epsilon, \exists j \notin\alpha: x_j > \epsilon \}, \nonumber
\end{aligned}
\end{equation}
so that $N_\alpha^\epsilon \subset O_\alpha^\epsilon$ and $R_\alpha^\epsilon  = O_\alpha^\epsilon \setminus N_\alpha^\epsilon$. Observe also that $\cone_\alpha = O_\alpha^0\setminus N_\alpha^0$. Thus, $\mu(R_\alpha^\epsilon) = \mu(O_\alpha^\epsilon) - \mu(N_\alpha^\epsilon)$, and $\mu(\cone_\alpha) = \mu(O_\alpha^0) - \mu(N_\alpha^0)$,  so that it is sufficient to show that 
\begin{equation}
\mu(N_\alpha^0) = \lim_{\epsilon\to 0}\mu(N_\alpha^\epsilon) ,
\quad \text{and}  \quad
\mu(O_\alpha^0) = \lim_{\epsilon\to 0}\mu(O_\alpha^\epsilon). \nonumber
\end{equation}
Notice that the $N_\alpha^\epsilon$'s and the $O_\alpha^\epsilon$'s form two increasing sequences of sets (when $\epsilon$ decreases), and that  $N_\alpha^0 = \bigcup_{\epsilon>0,\epsilon\in\mathbb{Q}} N_\alpha^\epsilon$, $O_\alpha^0 = \bigcup_{\epsilon>0,\epsilon\in\mathbb{Q}} O_\alpha^\epsilon$. This proves the desired result.
\end{proof}


We may now  make precise the  above heuristic
interpretation of the quantities $\mu(\cone_\alpha)$: the vector
$\mathcal{M}=\{ \mu(\mathcal{C}_{\alpha}):\; \emptyset \neq
\alpha\subset\{1,\; \ldots,\; d \}\}$ asymptotically describes the
dependence structure of the extremal observations. 
%
Indeed, by
Lemma~\ref{lem:limit_muCalphaEps}, and the discussion above, $
\epsilon$ may be chosen such that $R_\alpha^\epsilon$ is a
continuity set of $\mu$, while $\mu(R_\alpha^\epsilon)$ is
arbitrarily close to $\mu(\cone_\alpha)$.  Then, using the
characterization (\ref{eq:regularVariation}) of  $\mu$, 
the following asymptotic identity  holds true:
\begin{equation}
\begin{aligned}
\lim_{t\to\infty} t \PP\left( \|\mb V\|_\infty\ge t, V^j> \epsilon t~~ (j \in\alpha), V^j \le \epsilon t~~ (j\notin\alpha)\right) &=\mu(R_\alpha^\epsilon) \\ \nonumber
 &\simeq \mu(\cone_\alpha). \nonumber
\end{aligned}
\end{equation}
\begin{remark}
  \label{rk_approx_mu_n}
In terms of conditional probabilities, denoting $R = \|T(\mb X)\|$, where
  $T$ is the standardization map $\mb X\mapsto \mb V$,  we have
\begin{equation}
\nonumber  \PP(T(\mb X)\in r R_\alpha^\epsilon~|~ R>r) = 
\frac{r \PP(\mb V\in r R_\alpha^\epsilon)}{ r\PP(\mb V \in r([\mb 0,\mathbf{1}]^c)} \xrightarrow[r\to\infty]{}\frac{\mu(R_\alpha^\epsilon)}{\mu([\mb 0,\mathbf{1}]^c)},
\end{equation}
 as in~\eqref{eq:limitConditAngle}. In other terms, 
\begin{equation}
\begin{aligned}
\PP\left( V^j> \epsilon r~~ (j \in\alpha), V^j \le \epsilon r~~ (j\notin\alpha) ~\big\vert~ \|\mb V\|_\infty\ge r \right) &\xrightarrow[r\to\infty]{} C \mu(R_\alpha^\epsilon) \\ \nonumber
&~~~~~\simeq C\mu(\cone_\alpha), \nonumber
\end{aligned}
\end{equation}
where $C = 1/ \Phi(S_\infty^{d-1}) =1/\mu([\mb 0,\mathbf{1}]^c) $.
This clarifies the meaning of `large' and `small' in the heuristic
explanation given above. 
\end{remark}

\noindent {\bf Problem statement.} 
As explained above, our goal is to  describe the dependence on extreme
regions by investigating the structure of $\mu$ (or, equivalently,
that of $\Phi$). 
More precisely, the aim is twofold. First, recover a rough
approximation of the support of $\Phi$ based on the partition
$\{\Omega_\alpha, \alpha\subset\{1,\ldots,d\}, \alpha\neq
\emptyset\}$, that is, determine which $\Omega_\alpha$'s have
nonzero mass, or equivalently, which $\mu_\alpha's$ (\emph{resp.}
$\Phi_\alpha$'s) are nonzero. This support estimation is potentially
sparse (if a small number of $\Omega_\alpha$ have non-zero mass) and
possibly low-dimensional (if the dimension of the sub-cones
$\Omega_\alpha$ with non-zero mass is low).
The second objective is to 
investigate how the exponent measure $\mu$ spreads its mass on the
$\mathcal{C}_{\alpha}$'s, the theoretical quantity
$\mu(\mathcal{C}_{\alpha})$ indicating to which extent extreme
observations may occur in the `direction' $\alpha$ for $\emptyset
\neq \alpha \subset \{1,\; \ldots,\; d \}$. 
These two goals are achieved using empirical versions of
the angular measure defined in
Section~\ref{sec:classicEstimators}, evaluated on the
$\epsilon$-thickened rectangles $R_\alpha^\epsilon$.
Formally, we wish to recover the $(2^{d}-1)$-dimensional unknown
vector 
\begin{equation}
\label{eq:representation_M}
\mathcal{M}=\{ \mu(\mathcal{C}_{\alpha}):\; \emptyset \neq \alpha\subset\{1,\; \ldots,\; d \}\}
\end{equation}
 from $\mb X_1,\;
\ldots,\; \mb X_n\overset{i.i.d.}{\sim} \mb F$ and build an estimator
$\widehat{\mathcal{M}}$ such that
\begin{equation}
\nonumber
\vert\vert \widehat{\mathcal{M}} -\mathcal{M}
\vert\vert_{\infty} \;
{=} \; \sup_{\emptyset \neq \alpha \subset \{1,\; \ldots,\; d \}}\; \vert
\widehat{\mathcal{M}}(\alpha)- \mu(\mathcal{C}_{\alpha})\vert
\end{equation}
is small with large probability.  
In view of Lemma~\ref{lem:limit_muCalphaEps}, (biased) estimates of
$\mathcal{M}$'s components are built from an empirical version of 
the exponent measure, evaluated on the
$\epsilon$-thickened rectangles $R_\alpha^\epsilon$ (see Section~\ref{sec:classicEstimators} below). As a by-product, one obtains an estimate of the support of the limit measure $\mu$,
\begin{equation}
\bigcup_{\alpha:\; \widehat{\mathcal{M}}(\alpha)>0 }\mathcal{C}_{\alpha}. \nonumber
\end{equation}
 The results stated in the next section are non-asymptotic and sharp bounds are given by means of {\sc VC} inequalities tailored to low probability regions.

\subsection{Regularity Assumptions}\label{sec:RegularAssumptions} 
Beyond the existence of the limit measure $\mu$ (\ie~multivariate regular variation of $\mathbf{V}$'s distribution, see~\eqref{intro:regvar}), and thus, existence of an angular measure $\Phi$ (see (\ref{mu-phi})),
three additional assumptions are made, which are natural when estimation of the support of a distribution is considered. 

\begin{assumption}\label{hypo:continuous_margins}
The margins of $\mb X$ have continuous c.d.f., namely $F_j,~1 \le j \le d$ is continuous.
\end{assumption}
\noindent
 Assumption~\ref{hypo:continuous_margins}
is widely used in the context of non-parametric estimation of the dependence
structure (see \emph{e.g.} \cite{Einmahl2009}): it ensures that the transformed variables $V^j = (1 -
F_j(X^j))^{-1}$ (\emph{resp.} $U^j =  1 -F_j(X^j)$) have indeed a
standard Pareto distribution, $\P(V^j>x) = 1/x,~ x\ge 1$ (\emph{resp.}
the $U^j$'s are uniform variables). 

\bigskip
For any non empty subset $\alpha$ of $\{1,\; \ldots,\;d\}$, one denotes by $\ud x_\alpha$ the Lebesgue measure on ${\cal C}_\alpha$ and write $\ud x_\alpha  = \ud x_{i_1}\ldots\ud x_{i_k}$, when $\alpha=\{i_1, \ldots , i_k\}$. For convenience, we also write $\ud x_{\alpha\setminus{i}}$ instead of  $\ud x_{\alpha\setminus{\{i\}}}$.
\begin{assumption}\label{hypo:continuousMu}
Each component $\mu_\alpha$ of~\eqref{eq:decomp1} is absolutely continuous w.r.t.
Lebesgue measure $\ud x_\alpha$ on ${\cal C}_\alpha$. 
\end{assumption}
\noindent
Assumption~\ref{hypo:continuousMu} has a very convenient consequence 
regarding $ \Phi$: the fact that the exponent measure $\mu$ spreads no mass on subsets of the form 
$\{\mb x: \;\ninf{\mb x} \ge 1, x_{i_1} = \dotsb =  x_{i_r} \neq 0 \}$ with $r \ge 2$,
implies that the spectral measure $\Phi$ spreads no mass on edges $\{\mb x: \;\ninf{\mb x} = 1,  \; x_{i_1} = \dotsb =  x_{i_r} =1 \}$ with $r \ge 2~.$
This is summarized by the following result.
\begin{lemma}\label{lem:continuousPhi}
Under Assumption~\ref{hypo:continuousMu}, the following assertions holds true.
\begin{itemize}
\item  $ \Phi$ is concentrated on the (disjoint) edges
\begin{equation}
\begin{aligned}
   \Omega_{\alpha,i_0} = \{\mb x: \; \ninf{\mb x}  = 1,\; x_{i_0} = 1,~~& 0<  x_i < 1 ~~\text{~for~} i \in \alpha \setminus \{i_0\}\\ \nonumber
&x_i=0 ~~~~\text{~~~ for } i\notin \alpha ~~~~~~~\} \nonumber
\end{aligned}
\end{equation}
for $i_0\in\alpha$, $\emptyset \neq \alpha\subset\{1,\; \ldots,\; d \}$.
\item The restriction $\Phi_{\alpha,i_0}$ of $\Phi$ to
  $\Omega_{\alpha,i_0}$ is absolutely continuous \wrt~the Lebesgue measure $\ud x_{\alpha\setminus{i_0}}$ on the cube's edges, whenever $|\alpha|\ge 2 $.

\end{itemize}
\end{lemma}
\begin{proof}
  The first assertion straightforwardly results from the discussion above.  Turning to the
  second point, consider any  measurable
  set $D \subset \Omega_{\alpha,i_0}$ such that $\int_{D}\ud x_{\alpha \setminus i_0} = 0$. Then the
  induced truncated cone $\tilde D = \{ \mb v:~ \|\mb v\|_\infty \ge
  1, \mb v / \|\mb v\|_\infty \in D \}$ satisfies $\int_{\tilde D}\ud
  x_{\alpha} = 0$ and belongs to $\mathcal{C}_\alpha$. Thus,  by virtue of
  Assumption~\ref{hypo:continuousMu}, $\Phi_{\alpha,
    i_0}(D)=\Phi_{\alpha}(D) = \mu_\alpha(\tilde D) = 0$.
\end{proof}

\noindent
It follows from Lemma~\ref{lem:continuousPhi} that the angular
measure $\Phi$ decomposes as $\Phi = \sum_{\alpha} \sum_{i_0\in\alpha}
\Phi_{\alpha,i_0}$  and that there exist densities $ \frac{\ud
  \Phi_{\alpha,i_0}}{\ud x_{ \alpha \smallsetminus i_0}},~~
|\alpha|\ge 2,~i_0\in\alpha,$ such that for all $B \subset
\Omega_\alpha,~~ |\alpha| \ge 2$,
\begin{equation}
\label{eq:decomposePhi}
\Phi(B)~=~ \Phi_\alpha(B)~=~ \sum_{i_0\in\alpha} \int_{B\cap \Omega_{
    \alpha,i_0} } \frac{\ud \Phi_{\alpha,i_0}}{\ud x_{ \alpha
    \smallsetminus i_0}}(x) \ud x_{\alpha\setminus i_0}. 
\end{equation}
In order to formulate the next assumption, for $|\beta| \ge 2$, we set
\begin{equation}
\label{eq:supDensity}
M_\beta = 
~ \sup_{i \in\beta} ~~\sup_{x\in\Omega_{\beta,i}} ~~~~ \frac{\ud  \Phi_{\beta, i}}{\ud x_{\beta \setminus i}}(x).
\end{equation}

\begin{assumption}\label{hypo:abs_continuousPhi}({\sc Sparse Support})
  The angular density is uniformly bounded on $S^{d-1}_\infty$ ($\forall |\beta| \ge 2,~M_\beta < \infty$), and there exists a constant  $M>0$, such that we have $\sum_{|\beta| \ge 2} M_\beta < M$, where the sum is over subsets $\beta$ of $\{1,\ldots,d\}$ which contain at least two elements.
\end{assumption}

\begin{remark}
The constant $M$ is problem dependent. However, in the case where our representation $\mathcal{M}$ defined in \eqref{eq:representation_M} is the most informative about the angular measure, that is, when the density of $\Phi_\alpha$ is constant on $\Omega_\alpha$, we have $M \le d$: Indeed, in such a case, 
$M \le \sum_{|\beta| \ge 2} M_\beta |\beta| = \sum_{|\beta| \ge 2} \Phi(\Omega_\beta) \le \sum_\beta \Phi(\Omega_\beta) \le \mu([\mb 0,\mb 1]^c)$.
The equality inside the last expression comes from the fact that the Lebesgue measure of a sub-sphere $\Omega_\alpha$ is $|\alpha|$, for $|\alpha| \ge 2$. Indeed, using the notations defined in Lemma~\ref{lem:continuousPhi}, $\Omega_\alpha = \bigsqcup_{i_0 \in \alpha}\Omega_{\alpha,i_0}$, each of the edges $\Omega_{\alpha,i_0}$ being unit hypercube. 
Now, $\mu([\mb 0,\mb 1]^c) \le \mu(\{v,~ \exists j,~ v_j > 1\} \le d \mu(\{v,~v_1 >1\})) \le d$.

\noindent
Note that the summation $\sum_{|\beta| \ge 2} M_\beta |\beta|$ is smaller than $d$ despite the (potentially large) factors $|\beta|$. Considering $\sum_{|\beta| \ge 2} M_\beta$ is thus reasonable: in particular, $M$ will be small when only few $\Omega_\alpha$'s have non-zero $\Phi$-mass, namely when the representation vector $\mathcal{M}$ defined in \eqref{eq:representation_M} is sparse.
\end{remark}
\noindent Assumption~\ref{hypo:abs_continuousPhi} is naturally involved in the derivation of upper bounds on the error made when approximating $\mu(\cone_\alpha)$ by the empirical counterpart of $\mu(R_\alpha^\epsilon)$.
The estimation error bound derived in Section~\ref{sec:estimation} depends on the sparsity constant $M$.

\section{A non-parametric estimator of the subcones' mass : definition and preliminary results}
\label{sec:estimation}

In this section, an estimator $\widehat{\mathcal{M}}(\alpha)$ of each of the sub-cones' mass
$\mu(\cone_\alpha)$, $\emptyset\neq\alpha\subset\dd$, is
  proposed, based on observations  $\mb X_1,.\ldots, \mb X_n$, \iid~copies of $\mb X\sim \mb F$.
Bounds on the error $\vert\vert
  \widehat{\mathcal{M}}-\mathcal{M}\vert\vert_{\infty}$ are
  established. In the remaining of this paper, we work under
  Assumption~\ref{hypo:continuous_margins} (continuous margins, see
  Section~\ref{sec:RegularAssumptions}). 
Assumptions~\ref{hypo:continuousMu}~and~\ref{hypo:abs_continuousPhi}
are not necessary to prove a preliminary result on a class of
rectangles (Proposition~\ref{prop:g} and Corollary~\ref{cor:mu_n-mu}). However, they are required 
to bound the bias induced by the tolerance parameter
$\epsilon$ (in Lemma~\ref{lemma_simplex}, Proposition~\ref{prop_simplex} and in the main result, Theorem~\ref{thm-princ}).  
\subsection{A natural empirical version of $\mu$}
\label{sec:classicEstimators}
 Since the marginal distributions $F_j$ are unknown, we classically consider
the empirical counterparts of the $\mb V_i$'s, 
$\mb{\widehat  V}_i = (\widehat V_i^1, \ldots,\widehat
V_i^d)$, $1\le i\le n$, as standardized  variables obtained from a
rank transformation (instead of a probability integral
transformation),  
\[\mb{\widehat  V}_i = \left( ( 1- \widehat F_j
  (X_i^j))^{-1}\right)_{1 \le j \le d}~, \]
  where 
$\widehat F_j (x) = (1/n) \sum_{i=1}^n \mathbf{1}_{\{X_i^j < x\}}$.
We denote by $T$ (\emph{resp.} $\widehat T$) the standardization
(\textit{resp.} the empirical standardization), 
\begin{align}
\label{def:transform}
T(\mb x) = \left( \frac{1}{1- F_j (x^j)}\right)_{1\leq j\leq d}
\text{~~and~~}
\widehat T(\mb x) = \left( \frac{1}{1- \widehat F_j(x^j)}\right)_{1\leq j\leq d}. 
\end{align}
The empirical probability distribution of the rank-transformed data is then given by
\begin{align*}
\mathbb{\widehat P}_n=(1/n)\sum_{i=1}^n\delta_{\mb{\widehat{V}}_i}.
\end{align*}
Since for a $\mu$-continuity set $A$ bounded away from $0$,   $t~ \mathbb{P}\left( \mb V \in t A\right)  \to \mu(A)$ as $t \to \infty$, see~\eqref{eq:regularVariation}, a natural empirical version of $\mu$ is defined as 
\begin{align}\label{mu_n}
\mu_n(A) ~=~ \frac{n}{k} \widehat{\mathbb{P}}_n (\frac{n}{k}A) ~=~ \frac{1}{k}\sum_{i=1}^n \mathbf{1}_{\{\mb{\widehat{V}}_i \in \frac{n}{k} A\}}~.
\end{align}
 Here and throughout, we place ourselves in the asymptotic setting stipulating that $k = k(n) >0$ is such that $k \to \infty$ and $k = o(n)$ as $n \to \infty$.
The ratio $n/k$ plays the role of a large radial threshold.
Note that this estimator is commonly used in the field of
non-parametric estimation of the dependence structure, see \textit{e.g.}
\cite{Einmahl2009}.

\subsection{Accounting for the  non asymptotic nature of  data:
  $\epsilon$-thickening.}

 Since the cones $\mathcal{C}_\alpha$ have zero Lebesgue measure, 
and since, under Assumption~\ref{hypo:continuous_margins}, the margins are
continuous, the cones  are not likely to receive any empirical mass, 
 so that  simply counting points in $\frac{n}{k}\mathcal{C}_\alpha$ is not an
option: with probability one, only
the largest dimensional cone (the central one, corresponding to
$\alpha= \{1,\ldots,d\})$ will be hit. 
In view of Subsection~\ref{sec:decomposMu} and
Lemma~\ref{lem:limit_muCalphaEps}, 
 it is natural to introduce a 
tolerance parameter $\epsilon>0$ and to approximate the asymptotic mass
of $\mathcal{C}_\alpha$ with the non-asymptotic mass of
$R_\alpha^\epsilon$. We thus define the non-parametric estimator $\widehat{M}(\alpha)$ of
$\mu(\cone_\alpha)$ as
\begin{align}
\label{heuristic_mu_n}
\hatmass(\alpha) = \mu_n(R_\alpha^\epsilon), \qquad
\emptyset\neq\alpha\subset\dd.  
\end{align}
Evaluating  $\hatmass(\alpha)$  boils down (see~\eqref{mu_n})
to counting points in $(n/k)\,R_{\alpha}^{\epsilon}$, as illustrated in Figure~\ref{estimation_rect}. The estimate $\hatmass(\alpha)$ is thus a (voluntarily
$\epsilon$-biased) natural estimator of $\Phi(\Omega_\alpha) = \mu(\mathcal{C}_\alpha)$.
\begin{figure}[h]
  \centering
  \includegraphics[width = 0.7\textwidth]{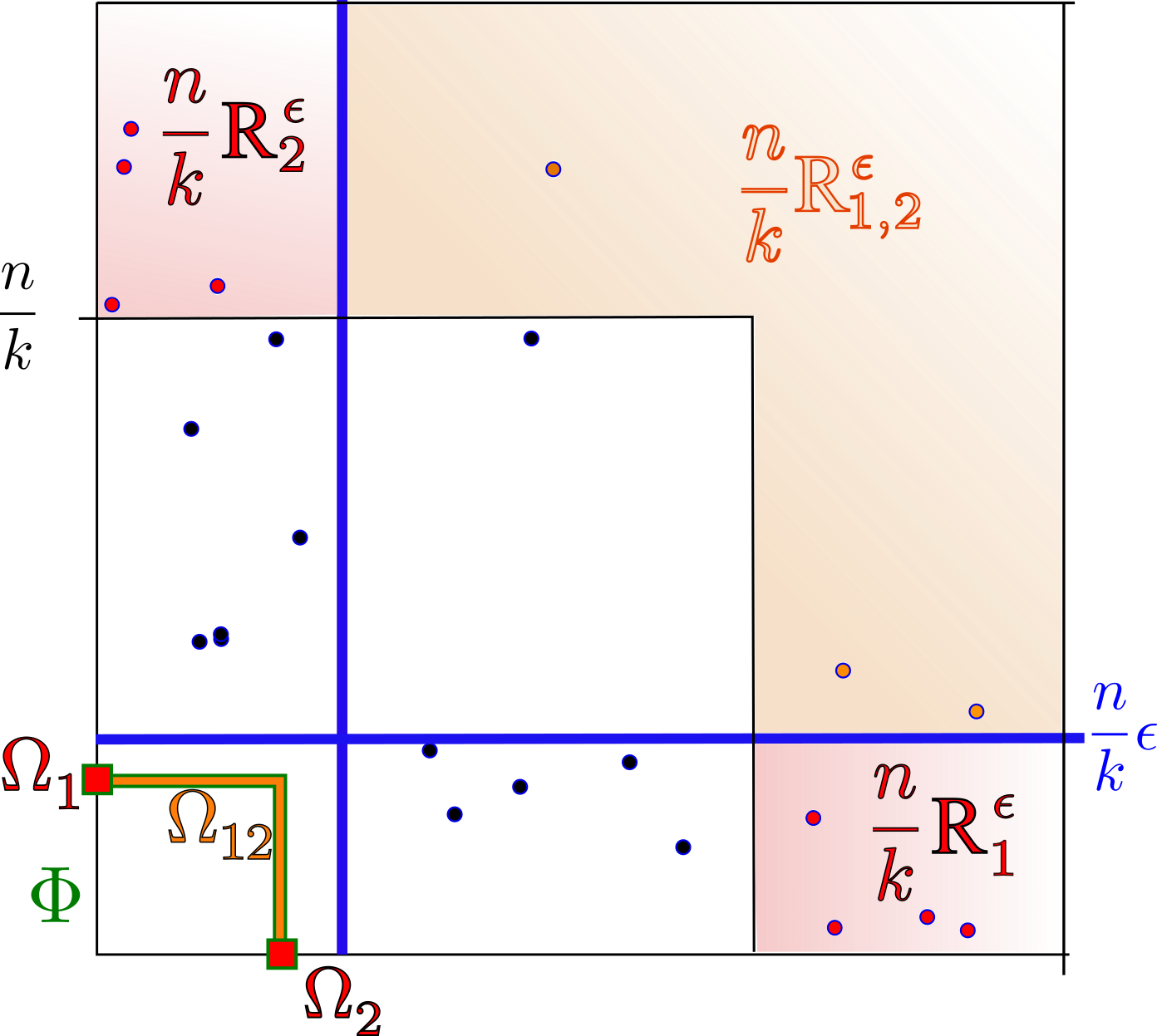}
  \caption{Estimation procedure}
  \label{estimation_rect}
\end{figure}

The coefficients $(\hatmass(\alpha))_{\alpha\subset\{1,\ldots,d\}}$ related to the cones $\mathcal{C}_\alpha$ constitute a summary representation of the dependence structure.  
This representation is sparse as soon as the $\mu_n^{\alpha,  \epsilon}$ are positive only for a few groups of features $\alpha$ (compared to the total number of groups or sub-cones, $2^d$ namely). It is  is low-dimensional as soon as each of these groups $\alpha$ is of small cardinality, or equivalently the corresponding sub-cones are low-dimensional compared with $d$.

In fact, $\hatmass(\alpha)$ is (up to a normalizing constant) an empirical version of the  conditional probability that $T(\mb X)$ belongs to the rectangle $ r R_\alpha^\epsilon$, given that $\|T(\mb X)\|$ exceeds  a large threshold $r$. Indeed, as explained in Remark~\ref{rk_approx_mu_n},  
\begin{align}\label{eq:interprete_mun_Pcondit}
\mathcal{M}(\alpha) = \lim_{r \to \infty} \mu([\mb 0,\mathbf{1}]^c)~~\mathbb{P}(T(\mb X)\in r R_\alpha^\epsilon ~|~ \|T(\mb X)\|\ge r) . 
\end{align}

The remaining of this section is devoted to obtaining non-asymptotic upper bounds on the error $\vert\vert \widehat{\mathcal{M}}-\mathcal{M}\vert\vert_{\infty}$. 
The main result is stated in Theorem~\ref{thm-princ}. 
Before all, notice that the error may be obviously decomposed as the sum of a stochastic term and a bias term inherent to the $\epsilon$-thickening approach:
\begin{align}
\vert\vert \widehat{\mathcal{M}}-\mathcal{M}\vert\vert_{\infty} &~=~\max_{\alpha} |
\mu_n(R_\alpha^\epsilon)-\mu(\mathcal{C}_\alpha)|\nonumber
\\&~\le~ ~\max_\alpha |\mu-\mu_n|(R_\alpha^\epsilon) ~+~ \max_\alpha|\mu(R_\alpha^\epsilon)-\mu(\mathcal{C}_\alpha)|~.\label{error_decomp} 
\end{align}
Here and beyond, for notational convenience, we simply denotes `$\alpha$'  for `$\alpha$ non empty subset of $\{1,\; \ldots,\;d\}$'. The main steps of the argument leading to Theorem~\ref{thm-princ} are  as follows. First, obtain a uniform upper bound on the error $|\mu_n - \mu|$ restricted to a well chosen VC class of rectangles (Subsection~\ref{sec:rectangles}), and deduce an uniform bound on  $|\mu_n - \mu|(R_\alpha^\epsilon)$ (Subsection~\ref{sec:boundErrorEpsilonCones}). Finally, using the regularity assumptions (Assumption~\ref{hypo:continuousMu} and Assumption~\ref{hypo:abs_continuousPhi}), bound the difference $|\mu(R_\alpha^\epsilon) - \mu(\cone_\alpha)|$ (Subsection~\ref{sec:boundMuEpsilonCones}).

\subsection{Preliminaries: uniform approximation over a VC-class of rectangles}
\label{sec:rectangles}
This subsection builds on the theory developed in \cite{COLT15}, where a non-asymptotic bound is stated on the estimation of the stable tail dependence function (STDF) defined in~\eqref{stdf1}. 
The STDF $l$ is related to the class of sets of the form $[\mb 0, \mb v]^c$ (or $[\mb u, \boldsymbol{\infty}]^c$ depending on which standardization is used), and an equivalent definition is 
\begin{align}
\label{stdf}
l(\mathbf{x}):= \lim_{t \to \infty} t \tilde F (t^{-1}\mathbf{x}) = \mu([\mb 0, \mb x ^{-1}]^c) 
\end{align}
\noindent
with $\tilde F (\mathbf{x}) = (1-F) \big( (1-F_1)^\leftarrow(x_1),\ldots, (1-F_d)^\leftarrow(x_d)  \big)$.
 Here the notation
$(1-F_j)^\leftarrow(x_j)$ denotes the quantity $\sup\{y\,:\; 1-F_j(y) \ge x_j\}$. Recall that the marginally uniform variable $\mb U$ is defined by  $U^j = 1-F_j(X^j)$ ($1\le j\le d$). Then  in terms of standardized variables $U^j$, 
\begin{align}
\label{def:tildeF}
\tilde F(\mb x) = \P\Big(\bigcup_{j=1}^d\{U^j< x_j\}\Big) = \P(\mb
U\in [\mb x, \boldsymbol{\infty}[^c) = \P(\mb V \in [\mb 0, \mb x^{-1}]^c). 
\end{align}

A natural estimator of $l$ is its empirical version defined as
follows,  see \cite{Huangphd}, \cite{Qi97}, \cite{Drees98}, \cite{Einmahl2006}, \cite{COLT15}:
\begin{align}\label{empir-Stdf}
l_n(\mathbf{x}) &= \frac{1}{k}~\sum_{i=1}^{n} \mathds{1}_{\{ X_i^1 \ge
  X^1_{(n-\lfloor kx_1 \rfloor+1)}  \text{~~or~~} \ldots \text{~~or~~}
  X_i^d \ge  X^d_{(n-\lfloor kx_d\rfloor+1)}  \}}~.
\end{align}
The expression is indeed suggested by the definition of $l$ in  (\ref{stdf}), with all distribution functions and  univariate quantiles replaced by their empirical counterparts, and with $t$  replaced by $n/k$.
The following lemma allows to derive alternative expressions for the empirical version of the STDF.  
\begin{lemma}
\label{lem:equivalence}
Consider the rank transformed variables 
$\mb{  \widehat U}_i = (\mb{ \widehat V}_i)^{-1} = ( 1- \widehat F_j
(X_i^j))_{1\leq j\leq d}$ for $i = 1, \ldots, n$. Then, for $(i,j)\in \{1,\ldots, n\}\times \{1,\ldots,d\}$,
with probability one,
$$
\widehat U_i^j \le \frac{k}{n} x_j^{-1} ~\Leftrightarrow~
\widehat V_i^j \ge \frac{n}{k} x_j ~\Leftrightarrow~
X_i^j \ge X_{(n-\lfloor  kx_j^{-1} \rfloor +1)}^j ~\Leftrightarrow~
U_i^j \le U_{(\lfloor kx_j^{-1}\rfloor)}^j~.
$$
\end{lemma}
The proof of Lemma~\ref{lem:equivalence} is standard and is provided in~\ref{appendix_proof} for completeness.
By Lemma~\ref{lem:equivalence}, the following alternative expression of $l_n(\mb x)$ holds true:
\begin{align}\label{empir-Stdf2}
l_n(\mb x)=\frac{1}{k} ~\sum_{i=1}^{n} \mathds{1}_{ \{U_i^1 ~\le~ U^1_{([kx_1])} \text{~~or~~} \ldots \text{~~or~~}  U_i^d ~\le~ U^d_{([kx_d])} \}}= \mu_n \left( [\mb 0, \mb x^{-1}]^c \right).
\end{align}
Thus, bounding the error $|\mu_n - \mu|([\mb 0,\mb x^{-1}]^c)$ is the same as
bounding $|l_n - l|(\mb x)$. 

Asymptotic properties of this empirical counterpart have been studied
in \cite{Huangphd}, \cite{Drees98}, \cite{Embrechts2000} and
\cite{dHF06} in the bivariate case, and \cite{Qi97},
\cite{Einmahl2012}. 
in the general multivariate
 case. 
In \cite{COLT15}, a non-asymptotic bound is established on the maximal deviation $$\sup_{0 \le \mb x \le T} |l(\mb x) - l_n(\mb x)|$$ for a fixed $T>0$, or equivalently on 
$$ \sup_{1/T \le \mb x } \left|\mu ( [\mb 0, \mb x]^c ) - \mu_n ( [\mb 0, \mb x]^c )\right|. $$
\noindent
The exponent measure $\mu$ is indeed easier to deal with when restricted to the class of sets of the form $[\mb 0, \mb x]^c$, which is fairly simple in the sense that it has finite VC dimension. 

In the present work,  an important step is to bound the error on the
class of $\epsilon$-thickened rectangles $R_\alpha^\epsilon$. This is achieved by
using a more general class $R(\mb x, \mb z, \alpha, \beta)$, which includes (contrary to the collection of sets $[\mb 0,\mb x]^c$) the $R_\alpha^\epsilon$'s . 
This flexible class is defined by
\begin{align}
\nonumber R(\mb x, \mb z, \alpha, \beta) &~=~ \Big\{ \mb y \in [0, \infty]^d,
~~ y_j \ge x_j ~~\text{ for } j \in \alpha, 
\\&~~~~~~~~~~~~~~~~~~~~~~~~~~ y_j < z_j ~~\text{ for } j \in \beta \quad \Big\},
~~~  \mb x, \mb z \in [0, \infty]^{d}. \label{set-R}
\end{align}
Thus,
\begin{align*}
\mu_n \left( R(\mb x, \mb z, \alpha, \beta) \right) & ~=~ \frac{1}{k}
\sum_{i=1}^n \mathds{1}_{\{ \widehat V_i^j ~\ge~ \frac{n}{k} x_j  \text{ for } j \in \alpha \text{ ~and~ } \widehat V_i^j ~<~\frac{n}{k} x_j  \text{ for } j \in \beta \}}~.
\end{align*}

\noindent
Then, define the  functional $g_{\alpha,\beta}$ (which plays the same role as the STDF) as follows:  
for $\mb x \in [0, \infty]^{d}\setminus\{\boldsymbol{\infty}\}$, $\mb z \in [0, \infty]^d$, $\alpha
\subset \{1,\ldots,d\} \setminus \emptyset$ and $\beta \subset
\{1,\ldots,d\}$, let
\begin{align}
&g_{\alpha, \beta}(\mb x, \mb z) ~~=~~ \lim_{t \to \infty} t \tilde
F_{\alpha, \beta}(t^{-1} \mb x, t^{-1} \mb z), \text{~~with}
\label{g-alpha} \\
&\tilde F_{\alpha, \beta}( \mb x,  \mb z) ~~=~~
 \mathbb{P} \left[ \left\{ U^j \le x_j ~~\text{ for } j \in \alpha
   \right\}~~\bigcap~~ \left\{U^j > z_j  ~~\text{ for } j \in\beta
   \right\}\right].
 \label{F-tilde-alpha}
\end{align}
 Notice that 
$\tilde F_{\alpha, \beta}( \mb x,  \mb z)$ is an extension of the
non-asymptotic approximation $\tilde F$ in~(\ref{stdf}).   
By~(\ref{g-alpha}) and~\eqref{F-tilde-alpha}, we have 
\begin{align*}
  g_{\alpha, \beta}(\mb x, \mb z) &= \lim_{t \to \infty} t \mathbb{P}
  \left[ \left\{ U^j \le t^{-1} x_j ~\text{ for } j \in \alpha
    \right\}~\bigcap~ \left\{  U^j > t^{-1} z_j  ~\text{ for } j
      \in\beta \right\}\right] \\&= \lim_{t \to \infty} t \mathbb{P}
  \left[ \mb V \in t R(\mb x^{-1}, \mb z^{-1}, \alpha,~\beta) \right]~,
\end{align*}
 so that using~\eqref{eq:regularVariation},
\begin{align}\label{g_with_mu}
g_{\alpha, \beta}(\mb x, \mb z) = \mu([ R(\mb x^{-1}, \mb z^{-1}, \alpha,~\beta)]).
\end{align}

%
The following lemma makes the relation between  $g_{\alpha,\beta}$ and  the angular measure
$\Phi$ explicit. Its proof is given in~\ref{appendix_proof}.
\begin{lemma}
\label{lem:g-alpha}
The function $g_{\alpha, \beta}$ can be represented as follows:
\begin{align*}
g_{\alpha, \beta}(\mb x, \mb z) = \int_{S^{d-1}} \left(\bigwedge_{j \in \alpha}{w_j x_j} - \bigvee_{j \in\beta} w_j z_j\right)_+~ \Phi(\ud \mb w)~,
\end{align*}
where $u\wedge v=\min\{ u,v \}$, $u\vee v=\max\{ u,v \}$ and $u_+=\max\{u, 0\}$ for any $(u,v)\in \mathbb{R}^2$.
\noindent
Thus, $g_{\alpha, \beta}$ is homogeneous and satisfies
\begin{align*}
|g_{\alpha, \beta}(\mb x, \mb z) - g_{\alpha, \beta}(\mb x', \mb z')| ~~\le~~  \sum_{j \in \alpha}|x_j - x_j'| ~+~  \sum_{j \in \beta}|z_j - z_j'|~,
\end{align*}
\end{lemma}
\begin{remark}
Lemma~\ref{lem:g-alpha} shows that the functional $g_{\alpha, \beta}$, which plays the same role as a  the STDF, enjoys  a Lipschitz property.
\end{remark}

We now define the empirical counterpart of $g_{\alpha, \beta}$
(mimicking that of the empirical STDF $l_n$ in (\ref{empir-Stdf}) ) by
\begin{align}
\label{def:gn}
g_{n, \alpha, \beta}(\mb x, \mb z) = \frac{1}{k}~\sum_{i=1}^{n}
\mathds{1}_{\{ X_i^j \ge X^j_{(n-\lfloor kx_j \rfloor+1)}  ~~\text{for}~
  j \in \alpha \text{~~~and~~~} X_i^j < X^j_{(n-\lfloor
    kx_j\rfloor+1)} ~~\text{for}~ j \in\beta \}}~.
\end{align}
As it is the case for the empirical STDF (see~(\ref{empir-Stdf2})), $g_{n,\alpha,\beta}$ has an alternative expression
\begin{align}\label{eq:gn2}
  \nonumber g_{n, \alpha, \beta}(\mb x, \mb z)&=\frac{1}{k} ~\sum_{i=1}^{n}
  \mathds{1}_{\{ U_i^j ~\le~ U^j_{([kx_j])} ~~\text{for}~ j \in \alpha
    \text{~~~and~~~} U_i^j ~>~U^j_{([kx_j])} ~~\text{for}~ j \in\beta \}}\\
  &= \mu_n \left( R(\mb x^{-1}, \mb z^{-1}, \alpha,~\beta)\right),
\end{align}
where the last equality comes from the equivalence $\widehat V_i^j \ge \frac{n}{k} x_j \Leftrightarrow U_i^j \le U_{(\lfloor kx_j^{-1}\rfloor)}^j$ (Lemma~\ref{lem:equivalence}) and from the expression $\mu_n(\point) = \frac{1}{k}\sum_{i=1}^n \mathds{1}_{\mb{\widehat{V}}_i \in \frac{n}{k} (\point)}$, definition~\eqref{mu_n}.

\noindent
The proposition below extends the result of \cite{COLT15}, by deriving an analogue upper bound on the maximal deviation 
$$
\max_{\alpha,\beta}\sup_{ 0 \le \mb x, \mb z \le T} |g_{\alpha, \beta}(\mb x, \mb z) - g_{n, \alpha, \beta}(\mb x, \mb z)|~,$$ 
or equivalently on
$$\max_{\alpha, \beta}~ \sup_{1/T \le \mb x, \mb z } \left|\mu ( R(\mb x, \mb z, \alpha, \beta)) - \mu_n (R(\mb x, \mb z, \alpha, \beta) )\right| ~.$$
Here and beyond 
we simply denote `$\alpha,\beta$'  for `$\alpha$ non-empty subset of $\{1,\ldots,d\}\setminus \emptyset$ and $\beta$ subset of
$\{1,\ldots,d\}$'. We also recall that comparison operators between two vectors (or
between a vector and a real number) are understood component-wise, \ie~ `$\mb x \le \mb z$' means `$x_j \le z_j$ for all $1\le j\le d$' and  for any real number $T$,  `$\mb x\le T$' means `$x_j \le T$ for all $1\le j\le d$'. 

\begin{proposition}
\label{prop:g}
Let $T \ge \frac{7}{2}(\frac{\log d}{k} + 1)$, and $\delta \ge e^{-k}$.  
Then there is a universal constant $C$, such that for each $n>0$, with probability at least $1 - \delta$,
\begin{align}
\label{ineq:g}
\max_{\alpha, \beta} \sup_{ 0 \le \mb x, \mb z \le T } &\left| g_{n, \alpha, \beta}(\mb x, \mb z) - g_{\alpha, \beta}(\mb x, \mb z) \right| ~\le~  Cd\sqrt{\frac{2T}{k}\log\frac{d+3}{\delta}} 
\\\nonumber &~~~~~~~~~~~~~~~~~~~~~~~+ \max_{\alpha, \beta}  \sup_{0 \le \mb x, \mb z \le 2T}\left|\frac{n}{k} \tilde F_{\alpha, \beta}( \frac{k}{n} \mb x, \frac{k}{n} \mb z)- g_{\alpha, \beta}(\mb x, \mb z)\right|.
\end{align}
The second term on the right hand side of the inequality is an asymptotic bias term which goes to $0$ as $n \to \infty$ (see Remark~\ref{rk:bias}).
\end{proposition}
The proof follows the same lines as that of Theorem 6 in \cite{COLT15} and is detailed in \ref{appendix_proof}. Here is the main argument.

The empirical estimator is based on the empirical measure of `extreme' regions, which are hit only with  low probability. It is thus enough to bound maximal deviations on such low probability regions. The key consists in choosing an adaptive VC class which only covers the latter regions (after standardization to uniform margins), namely a VC class composed of sets of the kind $\frac{k}{n} R(\mb x^{-1}, \mb z^{-1}, \alpha, \beta)^{-1}$. In \cite{COLT15}, VC-type inequalities have been established that incorporate $p$, the probability of hitting the class at all. Applying these inequalities to the particular class of rectangles gives the result.

\subsection{Bounding $|\mu_n - \mu|(R_\alpha^\epsilon)$ uniformly over $\alpha$}
\label{sec:boundErrorEpsilonCones}
The aim of this subsection is to exploit the previously established
bound on the deviations on rectangles, to obtain another uniform  bound for 
$|\mu_n-\mu|(R_\alpha^\epsilon)$, for $\epsilon >0$ and $\alpha \subset \dd$. In the remainder of the paper, $\bar \alpha$ denotes the complementary set of $\alpha$ in $\dd$.
Notice that directly from their definitions \eqref{eq:epsilon_Rectangle} and \eqref{set-R}, $R_\alpha^\epsilon$ and $R(\mb x, \mb z, \alpha, \beta)$ are linked by: 
\begin{align*}
R_\alpha^\epsilon = R(\boldsymbol \epsilon, \boldsymbol  \epsilon, \alpha, \bar \alpha) \cap [\mb 0, \mb 1]^c = R(\boldsymbol \epsilon, \boldsymbol  \epsilon, \alpha, \bar \alpha) \setminus R(\boldsymbol \epsilon, \boldsymbol{\tilde \epsilon}, \alpha, \{1,\ldots, d\})
\end{align*}
where $\boldsymbol{\tilde \epsilon}$ is defined by $\boldsymbol{\tilde \epsilon}_j = \mathds{1}_{j \in \alpha} + \epsilon \mathds{1}_{j \notin \alpha}$ for all $j \in \{1,\ldots,d\}$.
Indeed, we have: $R(\boldsymbol \epsilon, \boldsymbol  \epsilon, \alpha, \bar \alpha) \cap [\mb 0, \mb 1] = R(\boldsymbol \epsilon, \boldsymbol{\tilde \epsilon}, \alpha, \{1,\ldots, d\})$.
As a result, for $\epsilon < 1$,
$$\sup_{\epsilon \le \mb x, \mb z }|\mu_n-\mu|\left(R_\alpha^\epsilon\right) \le 2 \sup_{\epsilon \le \mb x, \mb z }|\mu_n-\mu|\left(R(\mb x,~ \mb z,~ \alpha,~ \bar \alpha)\right).$$
On the other hand, from~\eqref{eq:gn2} and \eqref{g_with_mu} we have 
\begin{align*}
\sup_{\epsilon \le \mb x, \mb z }|\mu_n-\mu|\left(R(\mb x,~ \mb z,~ \alpha,~ \bar \alpha)\right) ~=~ \sup_{0 \le \mb x, \mb z \le \epsilon^{-1}} \left| g_{n, \alpha, \bar \alpha}(\mb x, \mb z) - g_{\alpha, \bar \alpha}(\mb x, \mb z) \right|.
\end{align*}
\noindent
Then Proposition \ref{prop:g} applies with $T = 1/\epsilon$ and the following result holds true.
\begin{corollary}
\label{cor:mu_n-mu}
Let $0 < \epsilon \le (\frac{7}{2}(\frac{\log d}{k} + 1))^{-1}$, and $\delta \ge e^{-k}$.
Then there is a universal constant $C$, such that for each $n>0$, with probability at least $1 - \delta$,
\begin{align}
\max_{\alpha} \sup_{ \epsilon \le \mb x, \mb z } &\left| (\mu_n-\mu)(R_\alpha^\epsilon) \right| ~\le~  Cd\sqrt{\frac{1}{\epsilon k}\log\frac{d+3}{\delta}} 
\\\nonumber &~~~~~~~~~~~~~~~~~~~~~~~+ \max_{\alpha, \beta}  \sup_{0 \le \mb x, \mb z \le 2\epsilon^{-1}}\left|\frac{n}{k} \tilde F_{\alpha, \beta}( \frac{k}{n} \mb x, \frac{k}{n} \mb z)- g_{\alpha, \beta}(\mb x, \mb z)\right|.
\end{align}
\end{corollary}

\subsection{Bounding $|\mu(R_\alpha^\epsilon)-\mu(\mathcal{C}_\alpha)|$ uniformly over $\alpha$}
\label{sec:boundMuEpsilonCones}
In this section, an upper bound on the bias induced by handling
$\epsilon$-thickened rectangles is derived. 
As the rectangles $R_\alpha^\epsilon$ defined in \eqref{eq:epsilon_Rectangle} do not correspond to any set of angles on the sphere $S_\infty^{d-1}$,
we also define the {\it $(\epsilon, \epsilon')$-thickened cones}  
\begin{align}
\label{eq:epsilon_Cone}
\mathcal{C}_{\alpha}^{\epsilon, \epsilon'}~=\{\mb v \ge 0,~\|\mb v\|_\infty \ge 1,~ v_j > \epsilon \|\mb v\|_\infty  ~\text{ for } j \in \alpha,
~v_j \le \epsilon'  \|\mb v\|_\infty ~\text{ for } j \notin \alpha \} ,
\end{align}
which verify $\mathcal{C}_{\alpha}^{\epsilon, 0}\subset R_\alpha^\epsilon \subset \mathcal{C}_{\alpha}^{0, \epsilon}.$
Define the corresponding $(\epsilon, \epsilon')$-thickened sub-sphere
\begin{align}
\label{eq:epsilon_Sphere}
\Omega_{\alpha}^{\epsilon, \epsilon'} =~~ \big\{\mb x \in S^{d-1}_\infty , ~~ x_i >\epsilon ~~\text{ for } i\in\alpha~,~~  x_i \le \epsilon' ~~\text{ for
} i\notin \alpha   \big\} 
=~~ \mathcal{C}_\alpha^{\epsilon, \epsilon'} \cap S^{d-1}_\infty.
\end{align}
It is then possible to approximate rectangles $R_\alpha^\epsilon$ by the cones $\mathcal{C}_{\alpha}^{\epsilon, 0}$ and $\mathcal{C}_{\alpha}^{0, \epsilon}$, and then $\mu(R_\alpha^\epsilon)$ by $\Phi(\Omega_{\alpha}^{\epsilon, \epsilon'})$ in the sense that
\begin{align}
\label{eq:approx_Recctangle}
\Phi(\Omega_\alpha^{\epsilon, 0}) = \mu(\mathcal{C}_{\alpha}^{\epsilon, 0}) \le \mu(R_\alpha^\epsilon) \le \mu(\mathcal{C}_{\alpha}^{0, \epsilon}) = \Phi(\Omega_\alpha^{0, \epsilon}).
\end{align}

The next result (proved in \ref{appendix_proof}) is a preliminary step toward a bound on $|\mu(R_\alpha^\epsilon)-\mu(\cone_\alpha)|$. It is easier to use the absolute continuity of $\Phi$ instead of that of $\mu$, since the rectangles $R_\alpha^\epsilon$ are not bounded contrary to the sub-spheres $\Omega_\alpha^{\epsilon, \epsilon'}$. 
\begin{lemma}
\label{lemma_simplex}
For every $\emptyset \neq \alpha \subset \dd$ and $0 < \epsilon, \epsilon' < 1/2$, we have 
\begin{align*}
|\Phi(\Omega_\alpha^{\epsilon, \epsilon'}) - \Phi(\Omega_\alpha)| ~\le~ M |\alpha|^2 \epsilon ~+~ M d \epsilon'~.
\end{align*}
\end{lemma}
\noindent
 Now, notice that 
$$
\Phi(\Omega_\alpha^{\epsilon, 0}) - \Phi(\Omega_\alpha) \le \mu(R_\alpha^\epsilon) - \mu(\cone_\alpha) \le \Phi(\Omega_\alpha^{0, \epsilon}) - \Phi(\Omega_\alpha).
$$
We obtain the following proposition.
\begin{proposition}
\label{prop_simplex}
For every non empty set of indices $\emptyset \neq \alpha \subset \dd$ and $\epsilon > 0$,
\begin{align*}
|\mu(R_\alpha^\epsilon)-\mu(\cone_\alpha)| \le M d^2 \epsilon
\end{align*}
\end{proposition}

\subsection{Main result}
We can now state the main result of the paper, revealing the accuracy of the estimate \eqref{heuristic_mu_n}.
\begin{theorem}
\label{thm-princ}
There is an universal constant $C>0$ such that for every $n,~k,~\epsilon,~\delta$ verifying $\delta \ge e^{-k}$, $0 < \epsilon < 1/2$ and $ \epsilon \le (\frac{7}{2}(\frac{\log d}{k} + 1))^{-1}$,
the following inequality holds true with probability greater than $1-\delta$:
\begin{align*}
 \|\hatmass - \cal{M}\|_\infty 
&~\le~  C d \left( \sqrt{ \frac{1}{\epsilon k}\log\frac{d}{\delta}} + M d\epsilon \right) \\
&~~~~~~~~~~~+~ 4 \max_{\substack{\alpha ~\subset~ \dd\\ \alpha \neq \emptyset}}~~\sup_{0 \le \mb x, \mb z \le \frac{2 }{\epsilon}}\left|\frac{n}{k} \tilde F_{\alpha, \bar \alpha }( \frac{k}{n} \mb x, \frac{k}{n} \mb z)- g_{\alpha, \bar \alpha }(\mb x, \mb z)\right|.
\end{align*}
\end{theorem}
\noindent
Note that $\frac{7}{2}(\frac{\log d}{k} + 1)$ is smaller than $4$ as soon as $\log d / k < 1/7$, so that a sufficient condition on $\epsilon$ is $\epsilon < 1/4$.
The last term in the right hand side is a bias term which goes to zero as $n \to \infty$ (see Remark~\ref{rk:bias}).
The term $M d \epsilon$ is also a bias term, which represents the bias induced by considering $\epsilon$-thickened rectangles. It depends linearly on the sparsity constant $M$ defined in Assumption~\ref{hypo:abs_continuousPhi}. 
The value $k$ can be interpreted as the effective number of observations  used in the empirical estimate, \ie~the effective sample size for tail estimation. 
Considering classical inequalities in empirical process theory such as
VC-bounds, it is thus no surprise to obtain one  in $O(1/\sqrt k)$.
Too large values of $k$ tend to yield a large bias, whereas too small values of $k$ yield a large variance. For a more detailed discussion on the choice of $k$ we recommend \cite{ELL2009}.

The proof is based on decomposition~\eqref{error_decomp}. 
The first term $\sup_\alpha|\mu_n(R_\alpha^\epsilon)-\mu(R_\alpha^\epsilon)|$
on the right hand side of \eqref{error_decomp} is bounded using
Corollary~\ref{cor:mu_n-mu}, while Proposition~\ref{prop_simplex}
allows to bound the second one 
(bias term stemming from the tolerance parameter $\epsilon$). 
Introduce the notation 
\begin{align}
\label{eq:bias}
\text{bias}(\alpha,n,k,\epsilon) &= 4\sup_{0 \le \mb x, \mb z \le \frac{2}{\epsilon}}\left|\frac{n}{k} \tilde F_{\alpha, \bar \alpha }( \frac{k}{n} \mb x, \frac{k}{n} \mb z)- g_{\alpha, \bar \alpha }(\mb x, \mb z)\right|.
\end{align}
\noindent
With probability at least $1-\delta$,  
\begin{align*}
\forall~ \emptyset \neq \alpha\subset\dd,~~~~~~~~\\ \left|\mu_n(R_\alpha^\epsilon) - \mu(\mathcal{C}_\alpha)\right| ~\le~& 
Cd\sqrt{\frac{1}{\epsilon k}\log\frac{d+3}{\delta}}~+~ ~\text{bias}(\alpha,n,k,\epsilon) + M d^2  \epsilon~.
\end{align*}
The upper bound stated in
Theorem~\ref{thm-princ} follows.

\begin{remark}{{(\sc Thresholding the estimator})}
\label{rk:threshold}
In practice, we have to deal with non-asymptotic noisy data, so that many $\widehat{\mathcal{M}}(\alpha)$'s have very small values though the corresponding $\mathcal{M}(\alpha)$'s are null.
One solution is thus to define a threshold value, for instance a proportion $p$ of the averaged mass over all the faces $\alpha$ with positive mass, \ie~$\text{threshold} = p |A|^{-1} \sum_{\alpha} \widehat{\mathcal{M}}(\alpha)$ with $A = \{\alpha,~\widehat{\mathcal{M}}(\alpha) > 0\}$ .
Let us define $\widetilde{\mathcal{M}}(\alpha)$ the obtained thresholded $\widehat{\mathcal{M}}(\alpha)$. Then the estimation error satisfies:
\begin{align*}
\|\widetilde{\mathcal{M}} - \mathcal{M}\|_\infty &\le \| \widetilde{\mathcal{M}} - \widehat{\mathcal{M}} \|_\infty +  \| \widehat{\mathcal{M}} - \mathcal{M}\|_\infty \\
& \le p |A|^{-1} \sum_{\alpha} \widehat{\mathcal{M}}(\alpha) + \| \widehat{\mathcal{M}} - \mathcal{M} \|_\infty \\
& \le p |A|^{-1} \sum_{\alpha} \mathcal{M}(\alpha) + p |A|^{-1} \sum_{\alpha} |  \widehat{\mathcal{M}}(\alpha) - \mathcal{M}(\alpha) | \\&~~~~~~~~~~~~~~~~~~~~~~~~~~~~~~~~~~~~~~~~~~~~~~~~~~~~~~~~~~~+ \| \widehat{\mathcal{M}} - \mathcal{M} \|_\infty \\
& \le (p + 1) \| \widehat{\mathcal{M}} - \mathcal{M} \|_\infty + p |A|^{-1} \mu([0, 1]^c).
\end{align*}
It is outside the scope of this paper to study optimal values for $p$. However, Remark~\ref{rk:optim} writes the estimation procedure as an optimization problem, thus exhibiting a link between thresholding and $L^1$-regularization.
\end{remark}

\begin{remark}{(\sc Underlying risk minimization problems)}
\label{rk:optim}
Our estimate $\widehat{\mathcal{M}}(\alpha)$ can be interpreted as a solution of an empirical risk minimization problem inducing a conditional empirical risk $\widehat R_n$.
When adding a $L^1$ regularization term to this problem, we recover $\widetilde{\mathcal{M}}(\alpha)$, the thresholded estimate.

First recall that $\widehat{\mathcal{M}}(\alpha)$ is defined for $\alpha \subset \{1,\ldots,d\},~\alpha \neq \emptyset $ by $\widehat{\mathcal{M}}(\alpha) = 1/k \sum_{i=1}^{n} \mathds{1}_{\frac{k}{n} \hat{\mb V}_i \in R_\alpha^\epsilon}$. As $R_\alpha^\epsilon \subset [\mb 0, \mb 1]^c$, we may write
\begin{align*}
\widehat{\mathcal{M}}(\alpha) = \Big( \frac{n}{k} \mathcal{P}_n(\frac{k}{n} \| \hat{\mb V}_1 \| \ge 1) \Big) ~~ \Big( \frac{1}{n} \sum_{i=1}^n \frac{\mathds{1}_{\frac{k}{n} \hat{\mb V}_i \in R_\alpha^\epsilon} \mathds{1}_{\frac{k}{n} \| \hat{\mb V}_i \| \ge 1}}{ \mathcal{P}_n(\frac{k}{n} \| \hat{\mb V}_1 \| \ge 1)}\Big),
\end{align*}
where the last term is the empirical expectation of $Z_{n, i}(\alpha) = \mathds{1}_{\frac{k}{n} \hat{\mb V}_i \in R_\alpha^\epsilon}$ conditionnaly to the event $\{\|\frac{k}{n} \hat{\mb V}_1 \| \ge 1 \}$, and $\mathcal{P}_n(\frac{k}{n} \| \hat{\mb V}_1 \| \ge 1) = \frac{1}{n} \sum_{i=1}^n \mathds{1}_{\frac{k}{n} \|\hat{\mb V}_i\| \ge 1}$.
According to  Lemma~\ref{lem:equivalence}, for each fixed margin $j$, $\hat{V}_i^j  \ge \frac{n}{k}$ if, and only if $X_i^j \ge X_{(n-k+1)}^j$, which happens for $k$ observations exactly. Thus, $$\mathcal{P}_n(\frac{k}{n} \| \hat{\mb V}_1 \| \ge 1) = \frac{1}{n} \sum_{i=1}^n \mathds{1}_{\exists j,  \hat{\mb V}_i^j  \ge \frac{n}{k}} \in \left[\frac{k}{n}, \frac{dk}{n}\right].$$
If we define $\tilde k = \tilde k(n) \in [k, dk]$ such that $\mathcal{P}_n(\frac{k}{n} \| \hat{\mb V}_1 \| \ge 1) = \frac{\tilde k}{
   n}$, we then have  
\begin{align*}
\widehat{\mathcal{M}}(\alpha) &= \frac{\tilde k}{k} ~~ \left( \frac{1}{n} \sum_{i=1}^n \frac{\mathds{1}_{\frac{k}{n} \hat{\mb V}_i \in R_\alpha^\epsilon} \mathds{1}_{\frac{k}{n} \| \hat{\mb V}_i \| \ge 1}}{ \mathcal{P}_n(\frac{k}{n} \| \hat{\mb V}_1 \| \ge 1)}\right) \\
&= \frac{\tilde k}{k} ~~ \argmin_{m_\alpha > 0} \sum_{i=1}^n (Z_{n,i}(\alpha) - m_\alpha)^2 \mathds{1}_{\frac{k}{n} \| \hat{\mb V}_i \| \ge 1},
\end{align*}
Considering now the $(2^d -1)$-vector $\widehat{\mathcal{M}}$ and $\|.\|_{2, \alpha}$ the $L^2$-norm on $\mathbb{R}^{2^d-1}$, we immediatly have (since $k(n)$ does not depend on $\alpha$)
\begin{align}
\label{eq:optim_pb}
\widehat{\mathcal{M}} = \frac{\tilde k}{k} \argmin_{m \in \mathbb{R}^{2^d-1}} \widehat{R_n}(m),
\end{align}
where $\widehat{R_n}(m) = \sum_{i=1}^n \|Z_{n,i} - m\|_{2,\alpha}^2 \mathds{1}_{\frac{k}{n} \| \hat{\mb V}_i \| \ge 1}$ is the $L^2$-empirical risk of $m$, restricted to extreme observations, namely to observations $\mb X_i$ satisfying $\| \hat{\mb V}_i \| \ge \frac{n}{k}$. Then, up to a constant $\frac{\tilde k}{k} = \Theta(1)$, $\widehat{\mathcal{M}}$ is solution of an empirical conditional risk minimization problem. 
\noindent
Define the non-asymptotic theoretical risk $R_n(m)$ for $m \in \mathbb{R}^{2^d-1}$ by $$R_n(m) = \mathbb{E}\left[ \|Z_n - m\|_{2,\alpha}^2  \Big| \|\frac{k}{n} \mb V_1\|_\infty \ge 1\right]$$
with $Z_n:=Z_{n,1}$. Then one can show (see~\ref{appendix_proof}) that $Z_n$, conditionally to the event $\{\|\frac{k}{n} \mb V_1\| \ge 1\}$, converges in distribution to a variable $Z_\infty$ which is a multinomial distribution on $\mathbb{R}^{2^d-1}$ with parameters $(n=1, p_\alpha = \frac{\mu(R_\alpha^\epsilon)}{\mu([\mb 0, \mb 1]^c)}, \alpha \in \{1,\ldots,n\}, \alpha \neq \emptyset)$. In other words, 
$$\mathbb{P}(Z_\infty(\alpha) = 1) = \frac{\mu(R_\alpha^\epsilon)}{\mu([\mb 0, \mb 1]^c)}$$ for all $\alpha \in \{1,\ldots,n\}, \alpha \neq \emptyset $, and $\sum_\alpha Z_\infty(\alpha) = 1$.
Thus $R_n(m)\to R_\infty(m):=\mathbb{E}[\|Z_\infty - m \|_{2,\alpha}^2]$, which is the asymptotic risk. Moreover, the optimization problem
\begin{align*}
\min_{m \in \mathbb{R}^{2^d-1}} R_\infty(m)  
\end{align*}
admits $m = (\frac{\mu(R_\alpha^\epsilon)}{\mu([\mb 0, \mb 1]^c)}, \alpha \subset \{1,\ldots,n\}, \alpha \neq \emptyset)$ as solution.

Considering the solution of the minimization problem \eqref{eq:optim_pb}, which happens to coincide with the definition of $\widehat{\mathcal{M}}$, makes then sense if the goal is 
to estimate $\mathcal{M}:= (\mu(R_\alpha^\epsilon), \alpha \in \{1,\ldots,n\}, \alpha \neq \emptyset)$.
As well as
considering thresholded estimators $\widetilde{\mathcal{M}}(\alpha)$, since it amounts (up to a bias term) to add a $L^1$-penalization term to the underlying optimization problem:
Let us consider
\begin{align*}
\min_{m \in \mathbb{R}^{2^d-1}} \widehat{R_n}(m) ~+~ \lambda \|m\|_{1, \alpha}
\end{align*}
with $\|m\|_{1,\alpha} = \sum_{\alpha} |m(\alpha)|$ the $L^1$ norm on $\mathbb{R}^{2^d-1}$. In this optimization problem, only extreme observations are involved. It is a well known fact that solving it is equivalent to soft-thresholding the solution of the same problem without the penality term -- and then, up to a bias term due to the \textbf{soft}-thresholding, it boils down to setting to zero features $m(\alpha)$ which are less than some fixed threshold $T(\lambda)$. This is an other interpretation on thresholding as defined in Remark~\ref{rk:threshold}.
\end{remark}

\section{Application to Anomaly Detection }
\label{sec:appliAD}
\subsection{Background on AD}
\label{sec:appliADBackground}
\noindent 
\textbf{What is Anomaly Detection ?}
From a machine learning perspective, AD can be considered as a specific classification task, where the usual assumption in supervised learning stipulating that the dataset contains structural information regarding all classes breaks down, see \cite{Roberts99}. This typically happens in the case of two highly unbalanced classes: the normal class is expected to regroup a large majority of the dataset, so that the very small number of points representing the abnormal class does not allow to learn information about this class.
\textit{Supervised} AD consists in training the algorithm on a labeled (normal/abnormal) dataset including both normal and abnormal observations. In the \textit{semi-supervised} context, only normal data are available for training. This is the case in applications where normal operations are known but intrusion/attacks/viruses are unknown and should be detected. In the \textit{unsupervised} setup, no assumption is made on the data which consist in unlabeled normal and abnormal instances. In general, a method from the semi-supervised framework may apply to the unsupervised one, as soon as the number of anomalies is sufficiently weak to prevent the algorithm from fitting them when learning the normal behavior. Such a method should be robust to outlying observations.

\noindent
\textbf{Extremes and Anomaly Detection.}
As a matter of fact, `extreme' observations are often more susceptible to be anomalies than  others.
In other words, extremal observations are often at the \textit{border} between normal and abnormal regions and play a very special role in this context. As the number of observations considered as extreme (\emph{e.g.} in a Peak-over-threshold analysis) typically constitute less than one percent of the data, a classical AD algorithm would tend to systematically classify all of them as abnormal: it is not worth the risk (in terms of ROC or precision-recall curve for instance) trying to be more accurate in low probability regions without adapted tools. Also, new observations outside the `observed support' are most often predicted as abnormal. However, false positives (\ie~false alarms) are very expensive in many applications (\eg~aircraft predictive maintenance). It is thus of primal interest to develop tools increasing precision (\ie~the probability of observing an anomaly among alarms) on such extremal regions.

\noindent
\textbf{Contributions.}
The algorithm proposed  in this paper provides a scoring function which ranks
extreme observations according to their supposed degree of abnormality. This method is complementary to other  AD
algorithms, insofar as  two algorithms (that described  here, together with any
other  appropriate AD algorithm) may be trained on the same dataset.
Afterwards, the input space may be divided into two regions -- an
extreme region and a non-extreme one-- so that a new
observation in the central region (\emph{resp.} in the extremal
region) would be classified as abnormal or
not according  to the scoring function issued by the generic algorithm
(\emph{resp.} the one  presented
here). 
The scope of our algorithm concerns  both semi-supervised and
unsupervised problems. Undoubtedly, as it consists in learning a
`normal' (\ie\ not abnormal) behavior in extremal regions, it is optimally efficient when
trained on `normal' observations only. 
However it also applies to  unsupervised situations. 
Indeed, it involves a non-parametric but relatively coarse estimation
scheme which prevents from over-fitting normal data or fitting anomalies.
As a consequence, this method is robust to outliers and also applies when the training dataset contains a (small) proportion of anomalies. 
%
%

\subsection{Algorithm:  Detecting Anomalies among Multivariate
  EXtremes (DAMEX)}\label{sec:algo}
The purpose of this subsection is to explain the heuristic behind the
use of multivariate EVT for Anomaly Detection, which is in fact a
natural way to proceed when trying to describe the dependence
structure of extreme regions. The algorithm is thus introduced in an
intuitive setup,  which matches the theoretical framework and results obtained in sections \ref{sec:framework} and \ref{sec:estimation}.
The notations are the same as above:  $\mb X = (X^1,\ldots,X^d)$ is a
random vector in $\rset^d$, with joint (\emph{resp.} marginal)
distribution $\mb F$ (\emph{resp.} $F_j$, $j =1,\ldots,d$)  and $\mb
X_1,.\ldots, \mb X_n \sim \mb F$ is an \iid  sample.
The first natural step to study the dependence between the margins $X^j$ is to
standardize them, and the choice of standard Pareto margins (with
\emph{c.d.f.} 
$x \mapsto 1/x$) is convenient: Consider thus the $\mb V_i$'s and $\mathbf{\widehat V}_i$'s as defined in Section~\ref{sec:framework}.
%
One possible strategy  to investigate the dependence structure of
extreme events is to characterize, for each subset of features $\alpha
\subset \{1,...,d\}$, the `correlation' of these features given that
one of them at least  is large and the others are small. Formally, we
associate to each such $\alpha$ a coefficient $\mathcal{M}(\alpha)$ 
reflecting the degree of dependence between the features $\alpha$.
%
This coefficient is to be proportional to the expected number of
points $\mb V_i$ above a large radial threshold ($\|\mb V\|_\infty >r$), verifying $V_i^j$ `large' for  $j \in\alpha$, 
while $V_i^j$ `small' for $j\notin \alpha$. 
In order to define the notion of `large' and `small', fix a (small)
tolerance parameter $0<\epsilon<1$. Thus, our focus is on the 
expected proportion of points `above a large radial threshold' $r$  which belong to
the truncated rectangles $R_\alpha^\epsilon $ defined in \eqref{eq:epsilon_Rectangle}. More precisely, our goal is to estimate the above
expected proportion, when the tolerance parameter $\epsilon$ goes to
$0$. 

The standard empirical approach  --counting the number of
points in the regions of interest-- leads to  estimates $\hatmass(\alpha) = \mu_n(R_\alpha^\epsilon)$ (see \eqref{heuristic_mu_n}), with $\mu_n$ the empirical version of $\mu$ defined in \eqref{mu_n}, namely:
\begin{align}
\label{heuristic_mu_n2}
\hatmass(\alpha) = \mu_n(R_\alpha^\epsilon) =  \frac{n}{k} \mathbb{\widehat P}_n \left ( \frac{n}{k} R_\alpha^\epsilon \right),
\end{align}
where we recall that $\mathbb{\widehat P}_n=(1/n)\sum_{i=1}^n\delta_{\widehat{V}_i}$ is the empirical probability distribution of the rank-transformed data, and
$k = k(n) >0$ is such that $k \to \infty$ and $k = o(n)$ as $n \to \infty$.
The ratio $n/k$ plays the role of a large radial threshold $r$. From our standardization choice, counting points in
$(n/k)\,R_\alpha^\epsilon$ boils down to selecting, for each feature $j\le d$, the `$k$ largest values' $X_i^j$
among $n$ observations. According to the nature of the extremal dependence,
a number between $k$ and $dk$ of observations are selected: $k$ in
case of perfect dependence, $dk$ in case of `independence', which
means, in the EVT framework, that the components may only be large one at a time. In
any case, the number of observations considered as extreme is proportional to $k$, whence the normalizing factor $\frac{n}{k}$. 

The coefficients
$(\hatmass(\alpha))_{\alpha\subset\{1,\ldots,d\}}$ associated
  with the cones $\mathcal{C}_\alpha$  constitute our
  representation of the dependence structure.  
  This representation is sparse as soon as the $\hatmass(\alpha)$  are positive only for a few groups of features $\alpha$
(compared with the total number of groups, or sub-cones,  $2^d - 1$). It
is  is low-dimensional as soon as each of these groups  has moderate
cardinality $|\alpha|$, \ie\ as soon as  the sub-cones  with
positive $\hatmass(\alpha)$ are low-dimensional relatively to $d$.

In fact, up to a normalizing constant,  $\hatmass(\alpha)$ is
an empirical version of the  probability that $T(\mb X)$ belongs
to the cone $\mathcal{C}_{\alpha}$, conditioned upon exceeding a large threshold. Indeed, for $r, n$ and $k$ sufficiently large, we have (Remark~\ref{rk_approx_mu_n} and \eqref{eq:interprete_mun_Pcondit}, reminding that $\mb V = T(\mb X)$) 
\begin{align*}
\hatmass(\alpha)\simeq C \mathbb{P}(T(\mb X)\in r R_\alpha^\epsilon ~|~ \|T(\mb X)\|\ge
r) . 
\end{align*}
Introduce  an `angular scoring function'
\begin{align}\label{eq:angularscoring}
w_n(\mb x) = \sum_{\alpha }\hatmass(\alpha) \mathds{1}_{\{\widehat T(\mb x) \in R_\alpha^\epsilon\}}.
\end{align}
For each fixed (new observation) $\mb x$, $w_n(\mb x)$ approaches 
the probability that the random variable $\mb X$ belongs to the same cone
as $\mb x$ in the transformed space. In short,  $w_n(\mb x)$ is an
empirical version of the probability that  $\mb X$ and $\mb x$ have
approximately the  same `direction'. 
For AD, the degree of `abnormality' of the  new observation $\mb x$ 
should be related both to $w_n(\mb x)$ 
and to the uniform norm $\|\widehat T(\mb x)\|_\infty$ (angular and radial
components). More precisely, for $\mb x$ fixed such that
$T(\mb x)\in R_\alpha^\epsilon$.
Consider the  `\textit{directional tail region}' induced by $\mb x$, 
$A_{\mb x} =  \{ \mb y  : T(\mb y) \in R_\alpha^\epsilon\,,\;\|T(\mb y)\|_\infty \ge \| T(\mb x)\|_\infty\}.$
 Then, if  $\|T(\mb x)\|_\infty$ is large enough, we have (using~\eqref{mu-phi}) that 
\begin{align*}
\mathbb{P}\left( \mb X \in A_{\mb x} \right)  &= \mathbb{P}\left(\mb V\in \|T(\mb x)\|_\infty R_\alpha^\epsilon\right)\\
&= \mathbb{P}\left(\|\mb V\| \ge \|T(\mb x)\|\right) ~~ \mathbb{P}\left(\mb V\in \|T(\mb x)\|_\infty R_\alpha^\epsilon ~|~ \|\mb V\| \ge \|T(\mb x)\|\right) \\
&\simeq C~ \mathbb{P}\left(\|\mb V\| \ge \|T(\mb x)\|\right)~ \hatmass(\alpha) \\
& =  C ~\|\widehat T(\mb x) \|_\infty^{-1} ~w_n(\mb x).
\end{align*}
 This yields the scoring function
\begin{align}
\label{def:scoring}
s_n(\mb x):=  \frac{w_n(\mb x)}{\|\widehat T(\mb x)\|_\infty}, 
\end{align}
which is thus (up to a scaling constant $C$) an empirical version of $\mathbb{P}(\mb X\in A_{\mb x})$: the smaller $s_n(\mb x)$, the more abnormal the point $\mb x$ should be considered.
As an illustrative example, Figure~\ref{DAMEX-2D} displays the level
sets of this scoring function, both in the transformed and the non-transformed input space, in the 2D situation. The data are simulated under a 2D logistic distribution with asymmetric parameters. 
\begin{figure}[h]
\centering
\includegraphics[scale=0.2331]{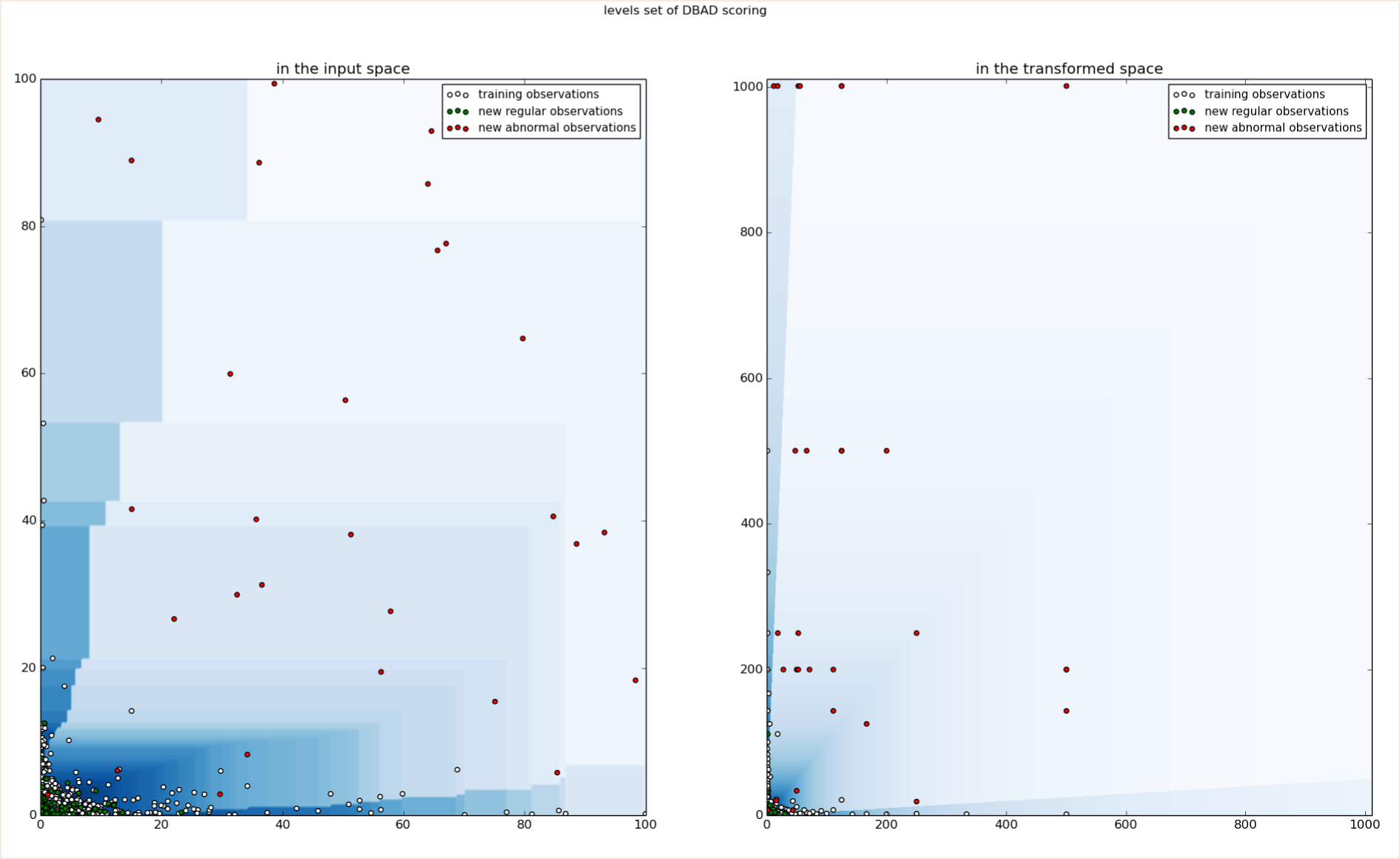}
\caption{Level sets of $s_n$ on simulated 2D data}
\label{DAMEX-2D}
\end{figure}

 This heuristic argument explains the following algorithm, referred to as {\it Detecting Anomaly with Multivariate EXtremes}  (DAMEX in abbreviated form). Note that this is a slightly modified version of the original DAMEX algorihtm empirically tested in \cite{AISTAT16}, where $\epsilon$-thickened sub-cones instead of $\epsilon$-thickened rectangles are considered. The proof is more straightforward when considering rectangles and performance remains as good.
The complexity is in $O( dn\log n + dn) = O(dn\log n)$, where the first term on the left-hand-side comes from  computing the $\widehat F_j(X_i^j)$ (Step 1) by sorting  the data (\emph{e.g.} merge sort). The second one arises from Step 2. 

\begin{center}
\fbox{
\begin{minipage}{0.95\linewidth}
\begin{algorithm} (DAMEX)
\label{DAMEX-algo}\\
{\bf Input:} parameters $\epsilon>0$,~~ $k = k(n)$,~~ $p\geq 0$.
\begin{enumerate}
\item Standardize \emph{via} marginal rank-transformation: $\mb{ \widehat V}_i:= \big (1/(1- \widehat F_j (X_i^j))\big)_{j=1,\ldots,d}$~.
\item Assign to each $\mb{\widehat V}_i$ the cone $R_\alpha^\epsilon$
  it belongs to.  
\item Compute $\hatmass(\alpha)$ from (\ref{heuristic_mu_n2})
   $\rightarrow$ yields: (small number of) cones with non-zero mass.
\item (Optional) Set to $0$ the $\hatmass(\alpha)$ below some small
  threshold defined in remark~\ref{rk:threshold} \wrt~$p$.
   $\rightarrow$ yields: (sparse) representation of the dependence
  structure 
 \begin{align}
 \label{phi_n}
\left\{\hatmass(\alpha):\; \emptyset\alpha\subset\{1,\ldots, d\}\right\}.
 \end{align}
\end{enumerate}
{\bf Output:} Compute the scoring function given
by~(\ref{def:scoring}), 
\begin{align*}
s_n(\mb x):= (1/\|\widehat T(\mb x)\|_\infty)
\sum_{\alpha }
\hatmass(\alpha) \mathds{1}_{\widehat T(\mb x) \in R_\alpha^\epsilon}.
\end{align*}
\end{algorithm}
\end{minipage}
}
\end{center}

Before investigating how the algorithm above empirically performs when applied to synthetic/real datasets, a few remarks are in order.

\begin{remark}({\sc Interpretation of the Parameters})
\label{rk_param_interpretation}
In view of 
(\ref{heuristic_mu_n2}), $n/k$ is the threshold above  which the data are
considered as extreme and $k$ is proportional to the number of such
data, a common approach in multivariate extremes.  
The tolerance parameter $\epsilon$ accounts for the  non-asymptotic nature of data. The
smaller $k$, the smaller $\epsilon$ shall be chosen. 
The additional angular mass threshold in step 4. acts as an additional
sparsity inducing parameter. Note that even without this additional
step (\ie\ setting $p=0$, the obtained representation for
real-world data (see Table~\ref{fig:wavedata-nb-faces})  is 
already sparse (the number of charges cones is significantly less than
$2^d$). 
\end{remark}
\begin{remark}({\sc Choice of Parameters})
\label{rk_param_choice}
A standard choice of parameters $(\epsilon,~ k ,~ p)$ is
respectively 
$(0.01, n^{1/2}, 0.1)$. 
However, there is no simple manner to choose optimally these parameters, as there is no simple way to determine how fast is the convergence to the (asymptotic) extreme behavior --namely how far in the tail appears the asymptotic dependence structure. Indeed, even though the  first term of the  error bound in Theorem~\ref{thm-princ} is  proportional, up to re-scaling, to $\sqrt{\frac{1}{\epsilon k} }+ \sqrt{\epsilon}$, which suggests choosing $\epsilon$ of order $ k^{-1/4}$, the unknown bias term perturbs the analysis and in practice, one obtains better results with the values above mentioned. 
In a supervised or semi-supervised framework (or if a small labeled dataset is available) these three parameters should be chosen by cross-validation.
In the unsupervised situation, a classical heuristic
(\cite{Coles2001}) is to choose $(k, \epsilon)$ in a stability region
of the algorithm's output: the largest $k$ (\emph{resp.} the larger
$\epsilon$) such that when decreased, the dependence structure remains
stable. This amounts to selecting as many  data as possible as being
extreme (\emph{resp. } in  low dimensional regions), within a stability
domain of the estimates, which exists under the primal assumption
\eqref{intro:assumption1} and in view of Lemma~\ref{lem:limit_muCalphaEps}. 
\end{remark}
\begin{remark} ({\sc Dimension Reduction})
If the extreme dependence structure is low dimensional, namely
concentrated on low dimensional cones $\mathcal{C}_\alpha$ -- or in other terms if only a
limited number of margins can be large together -- then most of the
$\widehat V_i$'s will be concentrated on the $R_\alpha^\epsilon$'s
such that  $|\alpha|$ (the dimension of the cone $\mathcal{C}_\alpha$)
is small; then the
representation of the dependence structure
 in (\ref{phi_n}) is both sparse and low dimensional.
\end{remark}

\begin{remark} ({\sc Scaling Invariance})
DAMEX produces the same result if the input data are transformed in such a way that the marginal order is preserved. In particular, any marginally increasing transform or any scaling as a preprocessing step does not affect the algorithm. It also implies invariance with respect to any change in the measuring units. This invariance property constitutes part of the strengh of the algorithm, since 
data preprocessing steps usually have a great impact on the overall performance and are of major concern in pratice.
\end{remark}


\section{Experimental results}
\label{sec:experiments}
\subsection{Recovering the support of the dependence structure of generated data}
Datasets of size $50000$ (respectively $100000$, $150000$) are  generated in $\mathbb{R}^{10}$ according to a popular multivariate extreme value
model, introduced by \cite{Tawn90},  namely a multivariate asymmetric
logistic distribution ($G_{log}$). 
The data have the following features: (i) they resemble `real life'
data, that is, the $X_i^j$'s are non
zero  and the transformed $\hat V_i$'s belong to the interior cone
$\mathcal{C}_{\{1,\ldots,d\}}$, (ii) the associated (asymptotic) exponent measure concentrates on
 $K$ disjoint cones $\{\mathcal{C}_{\alpha_m} , 1\le m\le K\}$.  
 For the sake of reproducibility, 
 $ G_{log}(\mb x) = \exp\{ - \sum_{m = 1}^K \left(\sum_{j \in \alpha_m}
     (|A(j)|x_j)^{ - 1/{w_{\alpha_m}}}\right)^{w_{\alpha_m}} \}, $
 where $|A(j)|$ is the cardinal of the set $\{\alpha\in D: j \in
 \alpha\}$ and where $w_{\alpha_m} = 0.1$ is a dependence parameter
 (strong dependence). 
The data are simulated using  Algorithm 2.2 in \cite{Stephenson2003}.
The subset of sub-cones $D$ charged by $\mu$ is randomly chosen (for each
fixed number of sub-cones $K$) and the purpose is to recover $D$ by Algorithm~\ref{DAMEX-algo}.
  For each $K$, $100$ experiments
are made and we consider  the  number of `errors', that is,      the number of
non-recovered or false-discovered sub-cones. Table~\ref{table:logevd} shows the averaged
numbers of errors  among the $100$ experiments. 
\begin{table}[h]
\centering
\scriptsize
\begin{tabular}{|l|lllllllllll|}
  \hline
  $\#$ sub-cones $K$       &    3 & 5    &  10   & 15   & 20    & 25  & 30   & 35    & 40    & 45    & 50 \\
  \hline
 Aver. $\#$ errors     & 0.02 & 0.65 & 0.95  & 0.45 & 0.49  & 1.35& 4.19 & 8.9  & 15.46  & 19.92  & 18.99 \\
   (n=5e4)     &&&&&&&&&&& \\
\hline
 Aver. $\#$ errors     & 0.00 & 0.45 & 0.36  & 0.21 & 0.13  & 0.43& 0.38 & 0.55  & 1.91  & 1.67  & 2.37 \\
   (n=10e4)     &&&&&&&&&&& \\

  \hline
 Aver. $\#$ errors     & 0.00  & 0.34 & 0.47 & 0.00 & 0.02  & 0.13& 0.13 & 0.31  & 0.39  & 0.59  & 1.77 \\
(n=15e4) &&&&&&&&&&& \\
  \hline
\end{tabular}
\caption{Support recovering on simulated data}
\label{table:logevd}
\end{table}
The results are very promising in situations where the number of sub-cones is moderate \emph{w.r.t.} the number of observations. 

\subsection{Sparse structure of extremes  (wave data)}
Our goal is here to verify that the two expected phenomena mentioned
in the introduction, \textbf{1-}~sparse dependence structure of extremes (small number
of sub-cones with non zero mass), \textbf{2-}~low dimension of the
sub-cones with non-zero mass,  do occur with real data. 
We consider wave
directions data provided by Shell, which consist of $58585$
measurements  $D_i$, $i\le 58595$ of wave directions between $0^{\circ}$ and $360^{\circ}$ at $50$ different
locations (buoys in North sea). The dimension is thus $50$. 
The angle $90^{\circ}$ being fairly
rare, we work with data obtained as $X_i^j = 1/(10^{-10} + |90-
D_i^j|)$, where $D_i^j$ is the wave direction at buoy $j$, time $i$. Thus,
$D_i^j$'s close to $90$ correspond to  extreme $X_i^j$'s.
Results in
Table~\ref{fig:wavedata-nb-faces}
show that 
the 
number of  sub-cones $\mathcal{C}_\alpha$ identified by Algorithm~\ref{DAMEX-algo}
is indeed small compared to the total number of sub-cones ($2^{50}$-1).
(Phenomenon \textbf{1} in the introduction section). 
Further, the dimension of these sub-cones is essentially moderate
(Phenomenon \textbf{2}):
respectively $93\%$, $98.6\%$ and  $99.6\%$
of the mass is affected to  sub-cones of dimension no greater  than $10$,
$15$ and $20$ respectively 
(to be compared with $d=50$).  Histograms displaying the mass repartition produced by Algorithm~\ref{DAMEX-algo} are given in Fig.~\ref{fig:wavedata-dim}.
\begin{figure}[h]
\centering
\includegraphics[scale=0.33]{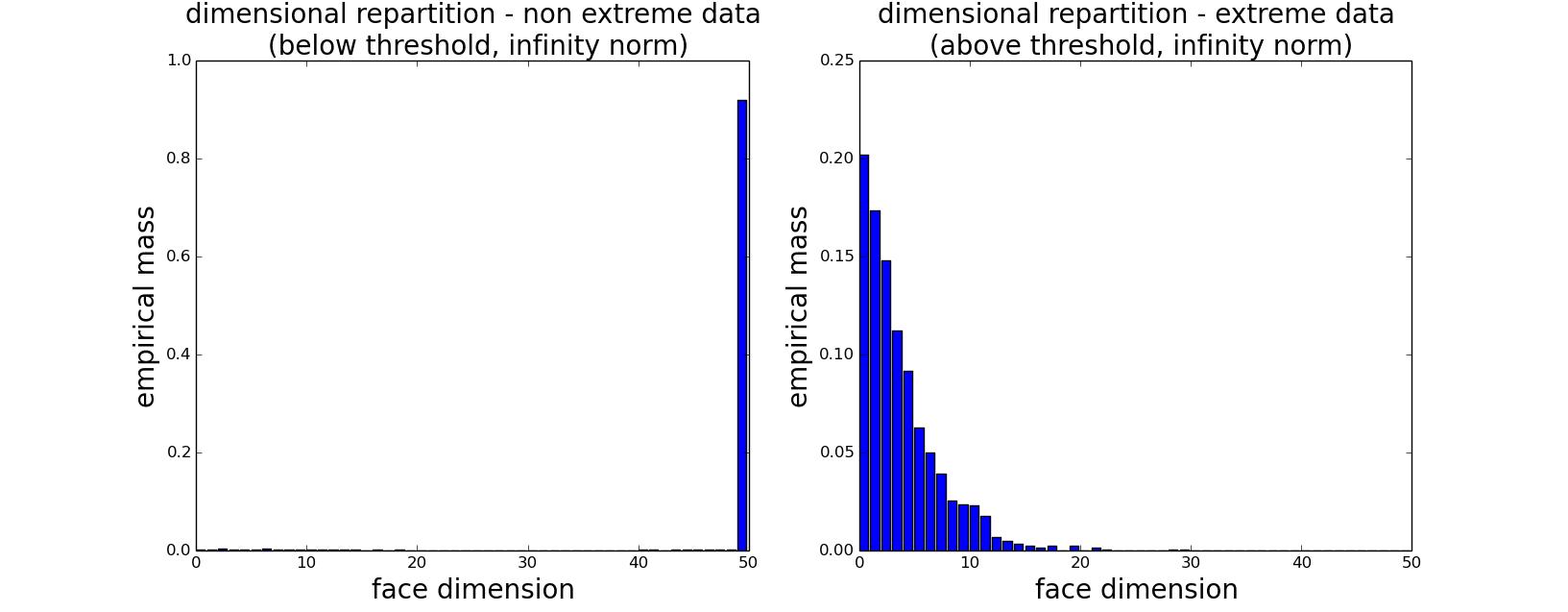}
\caption{sub-cone dimensions of wave data}
\label{fig:wavedata-dim}
\end{figure}

\begin{table}[h]
\centering
\footnotesize
\begin{tabular}{|l|cc|}
\hline
~ & non-extreme data & extreme data \\
\hline
nb of sub-cones with mass $>0$ ($p = 0$) & 3413 & 858 \\
idem after thresholding ($p = 0.1$) & 2 & 64 \\
idem after thresholding ($p = 0.2$) & 1 & 18 \\ 
\hline
\end{tabular}
\caption{Total number of sub-cones of wave data}
\label{fig:wavedata-nb-faces}
\end{table}

\subsection{Application to Anomaly Detection on real-world data sets}

 The main purpose of Algorithm~\ref{DAMEX-algo} is to build a 
 `normal profile' for 
extreme data, so as to distinguish between normal and ab-normal
extremes. 
In this section we evaluate its performance 
 and compare it with that of a standard AD algorithm, 
the Isolation Forest (iForest)
algorithm, which we chose in view of its
established high performance (\cite{Liu2008}). 
The two algorithms are trained and tested on the same datasets, the
test set being restricted to an  extreme region.
 Five reference AD datasets  are considered:
 \emph{shuttle}, \emph{forestcover}, \emph{http},
 \emph{SF} and \emph{SA} \footnote{These datasets are available for instance on http://scikit-learn.org/dev/ }. The experiments are performed in a
 semi-supervised framework (the training set consists of normal data). 
%
%

The \emph{shuttle} dataset is the fusion of the training and testing datasets
available in the UCI repository \cite{Lichman2013}. The data have $9$
 numerical attributes,  the first one being time. Labels from $7$ different classes are also
 available. Class $1$ instances are considered as normal, the others as anomalies. 
We use instances from all different classes but class $4$,  
which yields an anomaly ratio (class 1) of $7.17\%$. %
%

In the \emph{forestcover} data, also available at UCI
repository (\cite{Lichman2013}), the normal data are the  instances
from class~$2$ while instances from class $4$ are anomalies, other classes are omitted, 
so that the  anomaly ratio for this dataset is  $0.9\%$. 

The last three datasets belong to the KDD Cup '99 dataset
(\cite{KDD99}, \cite{Tavallaee2009}), produced by processing the
tcpdump portions of the 1998 DARPA Intrusion Detection System (IDS)
Evaluation dataset, created by MIT Lincoln Lab \cite{Lippmann2000}.
The artificial data was generated using a closed network and a wide
variety of hand-injected attacks (anomalies) to produce a large number
of different types of attack with normal activity in the background.
Since the original demonstrative purpose of the dataset concerns
supervised AD, the anomaly rate is very high ($80\%$), which is
unrealistic in practice, and inappropriate for evaluating the
performance on realistic data.  We thus take standard pre-processing
steps in order to work with smaller anomaly rates. For datasets
\emph{SF} and \emph{http} we proceed as described in
\cite{Yamanishi2000}: \emph{SF} is obtained by picking up the data
with positive logged-in attribute,
and 
focusing on the intrusion attack, which gives an anomaly proportion of
$0.48\%.$ 
The dataset \emph{http} is a subset of \emph{SF} corresponding to a
third feature equal to 'http'.
Finally, the \emph{SA} dataset  is obtained as in \cite{Eskin2002} by 
selecting all the normal data, together with a small proportion
($1\%$) of anomalies. 
%

Table~\ref{table:data} summarizes the characteristics of these
datasets. 
The thresholding parameter $p$ is fixed to $0.1$, the averaged mass of the non-empty sub-cones, while the parameters $(k,\epsilon)$ are standardly chosen as $(n^{1/2}, 0.01)$.
The extreme region on which the evaluation step is performed is chosen
as $\{\mb x:~ \|T(\mb x)\| > \sqrt{n} \}$, where $n$ is the  training
set's sample size. The ROC and PR curves are computed using only observations in the extreme region. This provides a precise evaluation of the two AD methods on extreme data.
For each of them, 20 experiments on random training and testing datasets are performed, yielding averaged ROC and Precision-Recall curves whose AUC are presented in Table~\ref{table:results-dbad+iforest-01}.
DAMEX significantly improves the performance (both in term of precision and of ROC curves) in extreme regions
for each dataset, as illustrated in figures \ref{fig:shuttle} and \ref{fig:forestcover}.

In Table~\ref{table:results-dbad+iforest-1}, we repeat the same experiments but with $\epsilon=0.1$.
This yields the same strong performance of DAMEX, excepting for \emph{SF}.
Generally, to large $\epsilon$ may yield over-estimated
$\hatmass(\alpha)$ for low-dimensional faces $\alpha$.
Such a performance gap between $\epsilon=0.01$ and $\epsilon=0.1$ can also be explained by the fact that anomalies may form a cluster which is wrongly include in some over-estimated `normal' sub-cone, when $\epsilon$ is too large. Such singular anomaly structure would also explain the counter performance of iForest on this dataset.


%

We also point out that for very small values of epsilon ($\epsilon \le 0.001$),
the performance of DAMEX significantly decreases on these datasets.
With such a small $\epsilon$, most observations belong to the central cone
(the one of dimension $d$) which is widely over-estimated, while the other cones are under-estimated.

The only case were using very small $\epsilon$ should be useful, is when the asymptotic behaviour is
clearly reached at level $k$ (usually for very large threshold $n/k$, \eg~$k=n^{1/3}$), or in the
specific case where anomalies clearly concentrate in low dimensional sub-cones: The use of a small $\epsilon$ precisely allows
to assign a high abnormality score to these subcones (under-estimation of the asymptotic mass), which yields better performances.


The averaged ROC curves and PR curves for the other datasets are gathered in \ref{appendix_exp}.
\begin{table}[h]
\centering
\footnotesize
\begin{tabular}{|l|cccccc|}
  \hline
   ~                   & shuttle & forestcover & SA     & SF     & http    \\
  \hline
  Samples total        & 85849   & 286048      & 976158 & 699691 & 619052  \\
  Number of features   & 9       & 54          & 41     & 4      & 3       \\
  Percentage of anomalies & 7.17    & 0.96        & 0.35   & 0.48   & 0.39 \\
\hline
\end{tabular}
\caption{Datasets characteristics}
\label{table:data}
\end{table}

\begin{table}[h]
\centering
\begin{tabular}{|l|cc|cc|}
  \hline
Dataset      &\multicolumn{2}{c|}{iForest}& \multicolumn{2}{c|}{DAMEX}\\
~            &AUC ROC       & AUC PR     &AUC ROC     &AUC PR       \\
shuttle      & 0.957        & 0.987      &$\mb{0.988}$&$\mb{0.996}$ \\
forestcover  & 0.667        & 0.201      &$\mb{0.976}$&$\mb{0.805}$ \\
http         & 0.561        & 0.321      &$\mb{0.981}$&$\mb{0.742}$ \\
SF           & 0.134        & 0.189      &$\mb{0.988}$&$\mb{0.973}$ \\
SA           & 0.932        &0.625       &$\mb{0.945}$&$\mb{0.818}$ \\ 
\hline
\end{tabular}
\caption{Results on extreme regions with standard parameters $(k,\epsilon) = (n^{1/2}, 0.01)$}
\label{table:results-dbad+iforest-01}
\end{table}

\begin{table}[h]
\centering
\begin{tabular}{|l|cc|cc|}
  \hline
Dataset      &\multicolumn{2}{c|}{iForest}& \multicolumn{2}{c|}{DAMEX}\\
~            &AUC ROC       & AUC PR     &AUC ROC     &AUC PR       \\
shuttle      & 0.957        & 0.987      &$\mb{0.980}$&$\mb{0.995}$ \\
forestcover  & 0.667        & 0.201      &$\mb{0.984}$&$\mb{0.852}$ \\
http         & 0.561        & 0.321      &$\mb{0.971}$&$\mb{0.639}$ \\
SF           & $\mb{0.134}$ & 0.189      &0.101       &$\mb{0.211}$ \\
SA           & 0.932        &0.625       &$\mb{0.964}$&$\mb{0.848}$ \\ 
\hline
\end{tabular}
\caption{Results on extreme regions with lower $\epsilon=0.1$}
\label{table:results-dbad+iforest-1}
\end{table}

\begin{figure}[h]
  \centering
  \includegraphics[width = \textwidth]{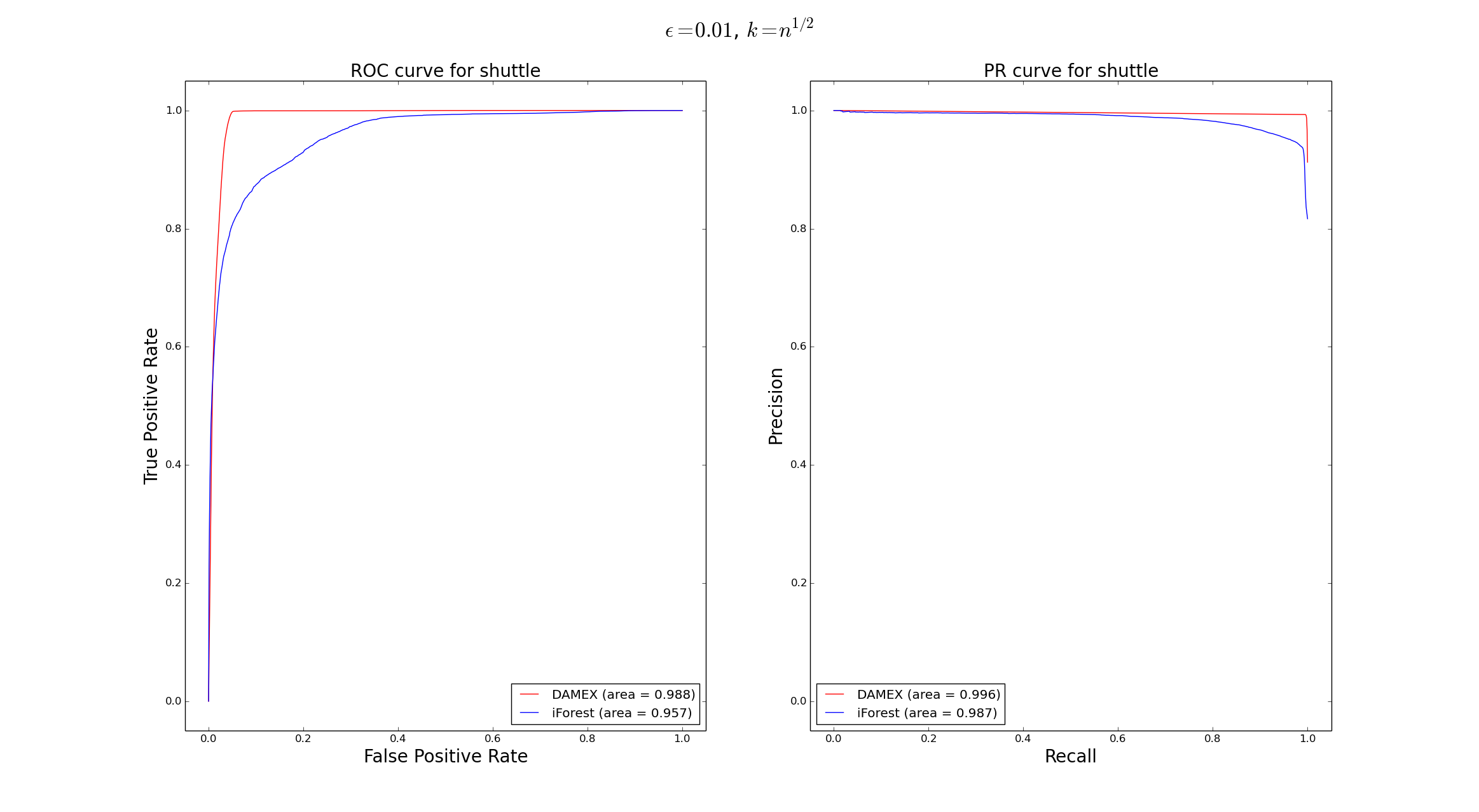}
  \caption{SF dataset, default parameters}
\label{fig:shuttle}
\end{figure}
\begin{figure}[h]
  \centering
  \includegraphics[width = \textwidth]{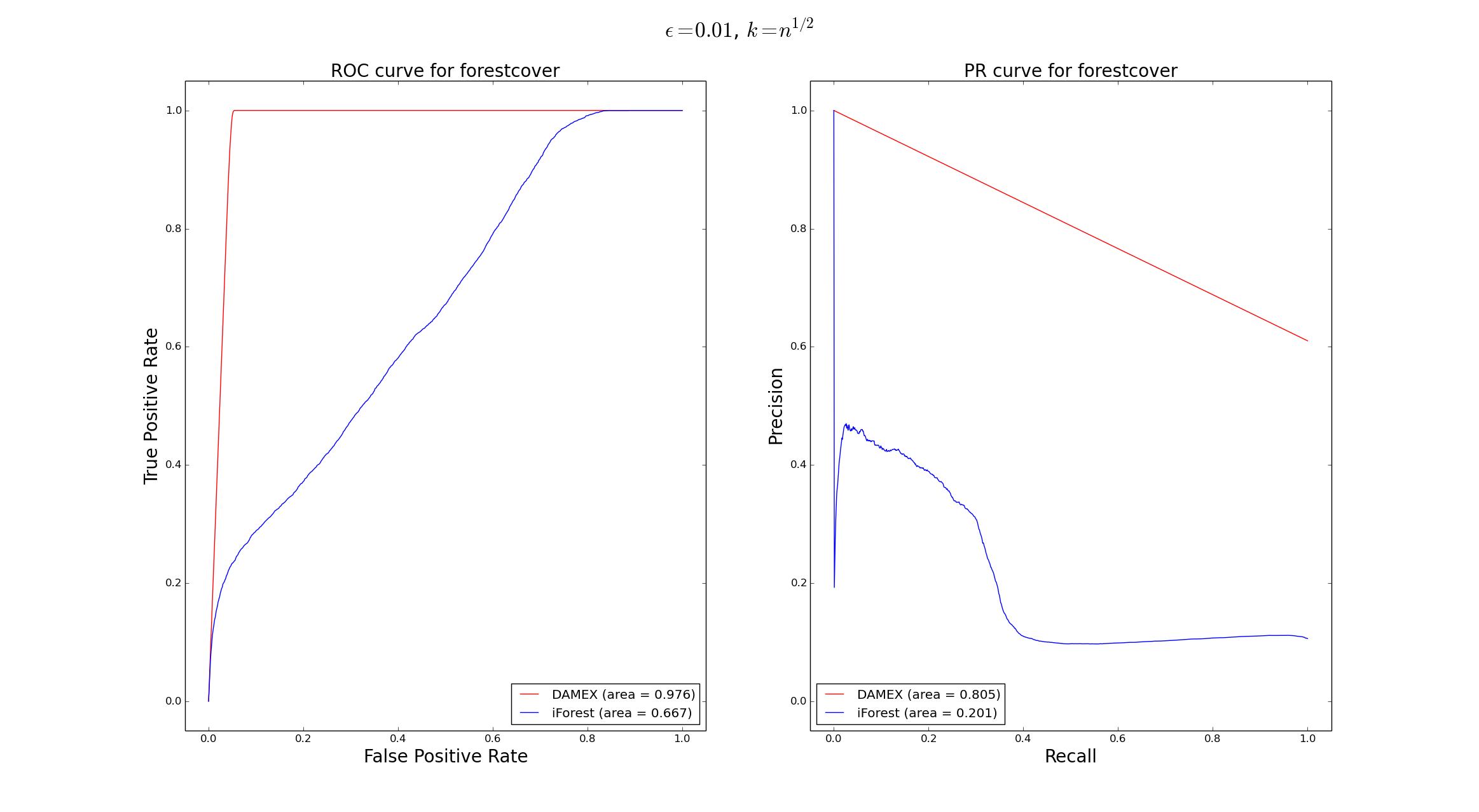}
  \caption{SF dataset, larger $\epsilon$}
\label{fig:forestcover}
\end{figure}

Considering the significant performance improvements on extreme data,
DAMEX may be combined with any standard AD algorithm to handle extreme
\emph{and} non-extreme data. This would improve the \emph{global}
performance of the chosen standard algorithm, and in particular
decrease the false alarm rate (increase the slope of the ROC curve's tangents near
the origin).  This combination can be done 
by splitting the input space between an
extreme region and a non-extreme one, then using
Algorithm~\ref{DAMEX-algo} to treat new observations that appear in
the extreme region, and the standard algorithm to deal with those which
appear in the non-extreme region.  


\section{Conclusion}
The contribution of this work is twofold. First, it brings advances in multivariate EVT by designing a statistical method that possibly exhibits a sparsity pattern in the dependence structure of extremes, while deriving non-asymptotic bounds to assess the accuracy of the estimation procedure. Our method is intended to be used as a preprocessing step to scale up multivariate extreme values modeling to high dimensional settings, which is currently one of the major challenges in multivariate EVT.
Since the asymptotic bias ($\text{bias}(\alpha,n,k, \epsilon)$ in eq.~\eqref{eq:bias}) appears as a separate term  in the bound established, no second order
assumption is required. One possible line of further research would be to make  such an assumption (\ie\,to assume that the bias itself is regularly varying), in order to choose $\epsilon$ adaptively with respect to $k$ and $n$ (see Remark~\ref{rk_param_choice}). This might also open up the possibility of de-biasing the estimation procedure (\cite{Fougeres2015}, \cite{Beirlant2015}). 
As a second contribution, this work extends the applicability of multivariate  EVT to the field of Anomaly Detection: a multivariate EVT-based algorithm which scores extreme observations according to their degree of abnormality is proposed. Due to  its  moderate complexity  --of order $d n \log n$--  this algorithm is suitable for the treatment of real word large-scale learning problems, and experimental results reveal a significantly increased performance on extreme regions compared with standard AD approaches.

\section*{Acknowledgements}
Part of this work has been supported by the industrial chair `Machine Learning for Big Data' from Telecom ParisTech, by the Ecole Normale Supérieure de Cachan and by the AGREED project from PEPS JCJC INS2I 2015.
 
\appendix
\section{Technical proofs}
\label{appendix_proof}
\subsection{Proof of Lemma \ref{lem:equivalence}}
For $n$ vectors $\mb v_1,\ldots,\mb v_n$ in $\mathbb{R}^d$, let us denote by $rank(v_i^j)$ the rank of $v_i^j$ among $v_1^j,\ldots,v_n^j$, that is $rank(v_i^j)=\sum_{k=1}^n \mathds{1}_{\{v_k^j \le v_i^j\}}$, so that $\hat F_j(X_i^j) = (rank(X_i^j)-1)/n$. For the first equivalence, notice that $\hat V_i^j = 1 / \hat U_i^j$.
For the others, we have both at the same time:
\begin{align*}
\hat V_i^j \ge \frac{n}{k} x_j  &~~\Leftrightarrow~~ 1-\frac{rank(X_i^j)-1}{n} ~\le~ \frac{k}{n}~x_j^{-1} 
\\ &~~\Leftrightarrow~~ rank(X_i^j) \ge n-kx_j^{-1} + 1 
\\ &~~\Leftrightarrow~~ rank(X_i^j) \ge n-\lfloor kx_j^{-1}\rfloor + 1 
\\ &~~\Leftrightarrow~~ X_i^j \ge X_{(n-\lfloor kx_j^{-1}\rfloor +1)}^j, 
\end{align*}
and
\begin{align*}
& X_i^j \ge X_{(n-\lfloor  kx_j^{-1} \rfloor +1)}^j ~\Leftrightarrow~
 rank(X_i^j) \ge n-\lfloor  kx_j^{-1} \rfloor+1 \\
&~~~~~~~~~~~~~~~~\Leftrightarrow~  rank( F_j(X_i^j)) \ge n-\lfloor
 kx_j^{-1}\rfloor+1 \qquad(\text{with probability one})\\ 
&~~~~~~~~~~~~~~~~\Leftrightarrow~  rank(1-F_j(X_i^j)) \le \lfloor kx_j^{-1}\rfloor\\
&~~~~~~~~~~~~~~~~\Leftrightarrow~  U_i^j \le U_{(\lfloor kx_j^{-1}\rfloor)}^j.
\end{align*}

\subsection{Proof of Lemma \ref{lem:g-alpha}}

First, recall that $g_{\alpha,\beta}(\mb x,\mb z) = \mu\big( R(\mb x^{-1},\mb z^{-1}, \alpha,\beta)\big)$, see \eqref{g_with_mu}. 
Denote by $\pi$ the transformation to pseudo-polar coordinates introduced
in Section~\ref{sec:framework}, 
\[
\begin{aligned}
  \pi: [0,\infty]^d \setminus\{\mb 0\} &\to  (0,\infty]\times S^{d-1}_\infty \\
\mb v & \mapsto (r,\boldsymbol{\theta}) = (\|\mb v\|_\infty, \|\mb v\|_\infty^{-1} \mb v). 
\end{aligned}
\]
Then, we have $\ud (\mu \circ\pi^{-1})= \frac{\ud r}{r^2}\ud
\Phi$ on $(0,\infty]\times S^{d-1}_\infty$. This classical result from EVT
comes from the fact that, for $r_0>0$ and $B\subset S^{d-1}_\infty$,  $\mu\circ\pi^{-1}\{ r\ge r_0, \boldsymbol{\theta}\in B\}
= r_0^{-1}\Phi(B)$, see \eqref{mu-phi}.
Then 
\[
\begin{aligned}
 g_{\alpha,\beta}(\mb x,\mb z)  & = 
\mu\circ\pi^{-1}\Big\{(r,\boldsymbol{\theta}): \quad \forall i \in\alpha,~ r \theta_i \ge
x_i^{-1}\;; \quad \forall j \in\beta, r \theta_j < z_j^{-1} \Big\} \\
&= \mu \circ \pi^{-1} \Big\{(r,\boldsymbol{\theta}): \quad  r  \ge \bigvee_{i\in\alpha}
(\theta_ix_i)^{-1}\;; \quad  r  <\bigwedge_{j\in\beta}(\theta_j z_j)^{-1} \Big\} \\
&= \int_{\boldsymbol{\theta}\in S^{d-1}_\infty} \int_{r>0} \mathds{1}_{r  \ge \bigvee_{i\in\alpha}
(\theta_ix_i)^{-1}}\;\mathds{1}_{r  <\bigwedge_{j\in\beta}(\theta_j
z_j)^{-1}} \frac{\ud r}{r^2} \ud \Phi(\boldsymbol{\theta}) \\ 
& = \int_{\boldsymbol{\theta}\in S^{d-1}_\infty}  \left(
\Big(\bigvee_{i\in\alpha} (\theta_ix_i)^{-1}\Big)^{-1}  - 
\Big(  \bigwedge_{j\in\beta}(\theta_jz_j)^{-1} \Big)^{-1} \right)_+
\ud \Phi(\boldsymbol{\theta}) \\ 
& = \int_{\boldsymbol{\theta}\in S^{d-1}_\infty}  \left(
\bigwedge_{i\in\alpha} \theta_ix_i   - 
 \bigvee_{j\in\beta}\theta_jz_j  \right)_+ \ud \Phi(\boldsymbol{\theta}), \\ 
\end{aligned}
\]
which proves the first assertion. 
To prove the Lipschitz property, notice first that, 
for any  finite sequence of real numbers  $c$ and $d$,
$\max_i c_i - \max_i d_i \le \max_i (c_i - d_i)$ and $\min_i c_i -
\min_i d_i \le \max_i (c_i - d_i)$. 
Thus for every $\mb x, \mb z \in [ 0, \infty]^d \setminus \{\boldsymbol{\infty}\}$ and $\theta \in S_\infty^{d-1}$:
\begin{align*}
  &\left(\bigwedge_{j \in \alpha}{\theta_j x_j} - \bigvee_{j \in\beta} \theta_j z_j\right)_+ - \left(\bigwedge_{j \in \alpha}{\theta_j x_j'} - \bigvee_{j \in\beta} \theta_j z_j'\right)_+ \nonumber\\
  &~~~~~~~~~~~~~~~~~~~~~~~~~~~~\le~
 \left[\left(\bigwedge_{j \in \alpha}{\theta_j x_j} -
      \bigvee_{j \in\beta} \theta_j z_j\right) - \left(\bigwedge_{j
        \in \alpha}{\theta_j x_j'} - \bigvee_{j \in\beta} \theta_j
      z_j'\right)\right]_+ \nonumber\\
&~~~~~~~~~~~~~~~~~~~~~~~~~~~~\le~
\left[\bigwedge_{j \in \alpha}{\theta_j x_j} - \bigwedge_{j \in
  \alpha}{\theta_j x_j'} ~+~ \bigvee_{j \in\beta} \theta_j z_j' -
\bigvee_{j \in\beta} \theta_j z_j\right]_+\nonumber\\
&~~~~~~~~~~~~~~~~~~~~~~~~~~~~\le~
    \left[\max_{j \in \alpha}(\theta_j x_j - \theta_j x_j')
      ~+~ \max_{j \in\beta} (\theta_j z_j' - \theta_j z_j) \right]_+ \\
 &~~~~~~~~~~~~~~~~~~~~~~~~~~~~\le~
\max_{j \in \alpha}\theta_j|{ x_j} - { x_j'}| ~+~ \max_{j \in\beta} \theta_j | z_j' -  z_j|
\end{align*}
Hence,
\begin{align*}
&|g_{\alpha,\beta}(\mb x, \mb z) - g_{\alpha,\beta}(\mb x', \mb z')|\\
&~~~~~~~~~~~~~~~~~~~~~~~\le~\int_{S^{d-1}_\infty} \left(\max_{j \in \alpha}\theta_j|{ x_j} - { x_j'}| ~+~ \max_{j \in\beta} \theta_j | z_j' -  z_j|\right)  \ud\Phi(\boldsymbol{\theta})~.
\end{align*}
Now, by \eqref{eq:integratePhiLambda} we have:
$$\int_{S^{d-1}_\infty} \max_{j \in \alpha}\theta_j|{ x_j} - { x_j'}| ~~ \ud\Phi(\boldsymbol{\theta}) = \mu([\mb 0, \mb{\tilde x}^{-1}]^c)$$ with $\mb{\tilde x}$ defined as $\tilde x_j = |x_j - x_j'|$ for $j\in\alpha$, and $0$ elsewhere.
It suffices then to write:
\begin{align*}
\mu([\mb 0, \mb{\tilde x}^{-1}]^c) &= \mu(\{y,~\exists j \in \alpha,~y_j \ge |x_j-x_j'|^{-1}\})\\
 &\le \sum_{j\in\alpha}\mu(\{y,~y_j \ge |x_j-x_j'|^{-1}\})\\
 &\le \sum_{j \in \alpha} |x_j-x_j'|~.
\end{align*}
Similarly, $\int_{S^{d-1}_\infty} \max_{j \in\beta} \theta_j | z_j' -  z_j|  ~~ \ud\Phi(\boldsymbol{\theta}) ~\le~ \sum_{j \in \beta} |z_j-z_j'|$.

%

\subsection{Proof of Proposition \ref{prop:g}}
The starting point is inequality (9) on p.7 in \cite{COLT15} which bounds the deviation of the empirical measure on extreme regions. 
Let $\mathcal{C}_n(\point)=\frac{1}{n} \sum_{i=1}^{n} \mathds{1}_{\{\mb Z_i \in \point \}}$ 
and $\mathcal{C}(\mathbf{x})=\mathbb{P}(\mb Z \in \point)$ be the empirical and true measures associated with a n-sample $\mb Z_1, \ldots, \mb Z_d$ of \iid~realizations of a random vector $\mathbf{Z}=(Z^1,\ldots,Z^d)$ with uniform margins on $[0,1]$. Then for any real number $\delta \ge e^{-k}$, 
with probability greater than $1 - \delta$,
\begin{align}
\label{Qialt2}
\sup_{0 \le \mathbf{x} \le T} \frac{n}{k} \left | \mathcal{C}_n(\frac{k}{n} [\mathbf{x},\boldsymbol{\infty}[^c) - \mathcal{C}(\frac{k}{n} [\mathbf{x},\boldsymbol{\infty}[^c)  \right| ~\le~ C d\sqrt{\frac{T}{k} \log{\frac{1}{\delta}}}~.
\end{align} 
Recall that with the above notations, $0 \le \mb x \le T$ means $0 \le x_j \le T$ for every $j$.
The proof of Proposition~\ref{prop:g} follows the same lines as in \cite{COLT15}. 
The cornerstone concentration inequality \eqref{Qialt2} has to be replaced with 
\begin{align}
\label{cornerstone-extension}
\nonumber \max_{\alpha, \beta} \sup_{\substack{  0 \le \mb x, \mb z \le T \\ \exists j \in \alpha, x_j \le T'}}
&\frac{n}{k} \left | \mathcal{C}_n \left(\frac{k}{n} R(\mb x^{-1}, \mb z^{-1}, \alpha,~\beta)^{-1}\right) - \mathcal{C}\left(\frac{k}{n} R(\mb x^{-1}, \mb z^{-1}, \alpha,~\beta)^{-1}\right)  \right|\\
&~~~~~~~~~~~~~~~~~~~~~~~~~~~~~~~~~~~~~~~~~ ~\le~ Cd \sqrt{\frac{dT'}{k} \log{\frac{1}{\delta}}}~.
\end{align}
\begin{remark}
Inequality \eqref{cornerstone-extension} is here written in its full generality, namely with a separate constant $T'$ possibly smaller than $T$. If $T' < T$, we then have a smaller bound (typically, we may use $T = 1/\epsilon$ and $T' = 1$). However, we only use \eqref{cornerstone-extension} with $T = T'$ in the analysis below, since the smaller bounds in $T'$ obtained (on $\Lambda(n)$ in \eqref{proof_decomp}) would be diluted (by $\Upsilon(n)$ in \eqref{proof_decomp}).
\end{remark}
\begin{proof}[Proof of (\ref{cornerstone-extension})]
Recall that for notational convenience we write `$\alpha,\beta$' for `$\alpha$ non-empty subset of $ \{1,\ldots,d\}$ and $\beta$ subset of $ \{1,\ldots,d\}$'.
The key is to apply Theorem 1 in \cite{COLT15}, with a VC-class which fits our purposes. Namely, consider
\begin{align*}
\mathcal{A} ~~=~~ \mathcal{A}_{T,T'} ~~&=~~ \bigcup_{\alpha, \beta} \mathcal{A}_{T,T',\alpha,\beta} \text{~~~~~with}\\
\mathcal{A}_{T,T',\alpha,\beta} ~~=~~& \frac{k}{n}\Big\{ R(\mb x^{-1}, \mb z^{-1}, \alpha,~\beta)^{-1}:~~ \mb x, \mb z \in \mathbb{R}^d ,~ 0 \le \mb x, \mb z \le T,\\
&~~~~~~~~~~~~~~~~~~~~~~~~~~~~~~~~~~~~~~~~~~~~~~~~~~~~~~~~\exists j \in \alpha, x_j \le T' \Big\}~,
\end{align*}
for $T,~T' > 0$ and $\alpha,~\beta \subset \dd,~\alpha \neq \emptyset$. $\mathcal{A}$ has VC-dimension $V_\mathcal{A} = d$, as the one considered in \cite{COLT15}. Recall in view of (\ref{set-R}) that 
\begin{align*}
R(\mb x^{-1}, \mb z^{-1}, \alpha,~\beta)^{-1} ~&=~ \Big\{ \mb y \in [0, \infty]^d , ~~y_j \le x_j ~~\text{ for } j \in \alpha,\\
&~~~~~~~~~~~~~~~~~~~~~~~~~  y_j > z_j ~~\text{ for } j \in\beta ~~\Big\} \\&=~ [\mb a, \mb b],
\end{align*}
with $\mb a$ and $\mb b$ defined by 
$a_j =   \left \lbrace \begin{array}{cc}
0 &  \text{for} ~ j \in \alpha \\
z_j  & \text{for} ~ j \in\beta \\
\end{array}
\right.$
and
$b_j =   \left \lbrace \begin{array}{cc}
x_j &  \text{for} ~ j \in \alpha \\
\infty  & \text{for} ~ j \in\beta \\
\end{array}
\right.$.
Since we have $\forall A \in \mathcal{A}, A \subset [\frac{k}{n} \mb T',~\boldsymbol{\infty}[^c$, the probability for a \rv~$\mb Z$ with uniform margins in $[0,1]$ to be in the union class $\mathbb{A} = \bigcup_{A \in \mathcal{A}}A$ is $\mathbb{P}(\mb Z \in \mathbb{A}) \le \mathbb{P}(\mb Z \in [\frac{k}{n}\mb T',~\boldsymbol{\infty}[^c) \le \sum_{j=1}^d \mathbb{P}(Z^j \le \frac{k}{n} T') \le \frac{k}{n}dT'$. 
Inequality \eqref{cornerstone-extension} is thus a direct consequence of Theorem 1 in \cite{COLT15}.
\end{proof}
\noindent
Define now the empirical version $\tilde F_{n, \alpha, \beta}$ of $\tilde F_{\alpha, \beta}$ (introduced in (\ref{F-tilde-alpha})) as 
\begin{align}
\label{def:Fn}
\tilde F_{n, \alpha, \beta}(\mb x, \mb z)  ~=~ \frac{1}{n} \sum_{i=1}^n \mathds{1}_{\{ U_i^j \le x_j ~~\text{for}~ j \in \alpha ~~\text{ and }~~  U_i^j > z_j ~~\text{for}~ j \in\beta \}}~ ,
\end{align}
so that 
$  \frac{n}{k} \tilde F_{n, \alpha, \beta}( \frac{k}{n} \mb x, \frac{k}{n} \mb z) ~=~ \frac{1}{k} \sum_{i=1}^n 
\mathds{1}_{\{ U_i^j \le \frac{k}{n} x_j ~~\text{for}~ j \in \alpha ~~\text{ and }~~ U_i^j > \frac{k}{n} z_j   ~~\text{for}~ j \in\beta \}}.$
Notice that the $U_i^j$'s are  not observable (since $F_j$ is
unknown). In fact, $\tilde F_{n, \alpha, \beta}$ will be used as a substitute for $g_{n, \alpha, \beta}$ (defined in \eqref{def:gn}) allowing to handle uniform variables. This is illustrated by the following lemmas. 

\begin{lemma}[Link between $g_{n, \alpha, \beta}$ and $\tilde F_{n, \alpha, \beta}$]
\label{lem:gn-Fn}
The  empirical version of $\tilde F_{\alpha, \beta}$ and that of $g_{\alpha, \beta}$ are related \emph{via}
\begin{align*}
g_{n, \alpha, \beta}(\mb x, \mb z)~=~\frac{n}{k} \tilde F_{n, \alpha, \beta}\left( \left(U_{(\lfloor kx_j\rfloor)}^j\right)_{j \in \alpha} , \left(U_{(\lfloor kz_j\rfloor)}^j\right)_{j \in \beta}\right),
\end{align*}
\end{lemma}

\begin{proof}
Considering the definition in (\ref{def:Fn}) and (\ref{eq:gn2}), both sides are equal to $\mu_n(R(\mathbf{x}^{-1}, \mathbf{z}^{-1}, \alpha,\beta))$. 
\end{proof}

\begin{lemma}[Uniform bound on $\tilde F_{n, \alpha, \beta}$'s deviations]
\label{lem:Fn-tildeF}
 For any finite  $T>0$, and $\delta\ge e^{-k}$,  with probability at least $1-\delta$, the  deviation of $\tilde F_{n, \alpha, \beta}$ from  $\tilde F_{\alpha, \beta}$ is uniformly bounded: 
 
\begin{align*}
\max_{\alpha, \beta} \sup_{ 0 \le \mb x, \mb z \le T}
\left| \frac{n}{k} \tilde F_{n, \alpha, \beta}( \frac{k}{n} \mb x, \frac{k}{n} \mb z)-\frac{n}{k} \tilde F_{\alpha, \beta}( \frac{k}{n} \mb x, \frac{k}{n} \mb z) \right| \le Cd\sqrt{\frac{T}{k}\log{\frac{1}{\delta}}}~.
\end{align*}
\end{lemma}
\begin{proof}
Notice that 
\begin{align*}
&\sup_{ 0 \le \mb x, \mb z \le T} \left| \frac{n}{k} \tilde F_{n, \alpha, \beta}( \frac{k}{n} \mb x, \frac{k}{n} \mb z)- \frac{n}{k} \tilde F_{\alpha, \beta}( \frac{k}{n} \mb x, \frac{k}{n} \mb z) \right|  \\
&~= \sup_{ 0 \le \mb x, \mb z \le T} \frac{n}{k} \left|  \frac{1}{n}  \sum_{i=1}^n \mathds{1}_{ \mb U_i \in \frac{k}{n} R(\mb x^{-1}, \mb z^{-1}, \alpha,~\beta)^{-1} } -
   \mathbb{P} \left[ \mathbf{U} \in \frac{k}{n} R(\mb x^{-1}, \mb z^{-1}, \alpha,~\beta)^{-1}
   \right] \right|,
\end{align*}
and apply inequality (\ref{cornerstone-extension}) with $T'=T$.
\end{proof}
\begin{remark}
\label{rk:lem:Fn-tildeF_generalized}
Note that the following stronger inequality holds true, when using (\ref{cornerstone-extension}) in full generality, \ie~with $T'< T$.
 For any finite  $T, T'>0$, and $\delta\ge e^{-k}$,  with probability at least $1-\delta$,
\begin{align*}
\max_{\alpha, \beta} \sup_{\substack{  0 \le \mb x, \mb z \le T \\ \exists j \in \alpha, x_j \le T'}}
\left| \frac{n}{k} \tilde F_{n, \alpha, \beta}( \frac{k}{n} \mb x, \frac{k}{n} \mb z)-\frac{n}{k} \tilde F_{\alpha, \beta}( \frac{k}{n} \mb x, \frac{k}{n} \mb z) \right| \le Cd\sqrt{\frac{T'}{k}\log{\frac{1}{\delta}}}.
\end{align*}
\end{remark}

The following lemma is stated and proved in \cite{COLT15}.
\begin{lemma}[Bound on the order statistics of $\mb U$]
\label{lem:U-x} 
Let $\delta\ge e^{-k}$. For any finite positive number $T>0$ such that $T \ge 7/2((\log d)/k + 1)$, we have with probability greater than $1 - \delta$, 
\begin{align}
\label{lem:eq-Wellner}
\forall~ 1\le j \le d,~~~~~\frac{n}{k} U_{(\lfloor kT\rfloor )}^j ~\le~ 2T~,
\end{align}
and with probability greater than $1- (d+1)\delta$, 
\begin{align*}
\max_{1 \le j \le d}~ \sup_{0 \le x_j \le T} \left| \frac{\lfloor kx_j\rfloor }{k} - \frac{n}{k} U_{(\lfloor kx_j\rfloor )}^j  \right| ~\le~ C\sqrt{\frac{T}{k}\log{\frac{1}{\delta}}}~.
\end{align*}
\end{lemma}
\noindent
We may now proceed with the proof of Proposition \ref{prop:g}.
Using Lemma \ref{lem:gn-Fn}, we may write:
\begin{align}
\nonumber &\max_{\alpha, \beta}\sup_{ 0 \le \mb x, \mb z \le T }
 \left| g_{n, \alpha, \beta}(\mb x, \mb z) - g_{\alpha, \beta}(\mb x, \mb z) \right| \\
\nonumber           &~~~~~~~~~~=~ \max_{\alpha, \beta}  \sup_{ 0 \le \mb x, \mb z \le T } \left| \frac{n}{k} \tilde F_{n, \alpha, \beta} \left( \left(U_{(\lfloor kx_j\rfloor)}^j\right)_{j \in \alpha} , \left(U_{(\lfloor kz_j\rfloor)}^j\right)_{j \in \beta} \right) - g_{\alpha, \beta}(\mb x, \mb z) \right| \\
\label{proof_decomp}&~~~~~~~~~~\le~ \Lambda(n) ~+~ \Xi(n) ~+~ \Upsilon(n)~.
\end{align}
with:
\begin{align*}
 \Lambda(n) &~=~ \max_{\alpha, \beta} \sup_{ 0 \le \mb x, \mb z \le T } \bigg| \frac{n}{k} \tilde F_{n, \alpha, \beta} \left(\left(U_{(\lfloor kx_j\rfloor)}^j\right)_{j \in \alpha} , \left(U_{(\lfloor kz_j\rfloor)}^j\right)_{j \in \beta} \right) 
\\&~~~~~~~~~~~~~~~~~~~~~~~~~~~~~~~~~~~~~~~- \frac{n}{k} \tilde F_{\alpha, \beta} \left(\left(U_{(\lfloor kx_j\rfloor)}^j\right)_{j \in \alpha} , \left(U_{(\lfloor kz_j\rfloor)}^j\right)_{j \in \beta} \right)  \bigg|\\
\Xi(n) &~=~\max_{\alpha, \beta} \sup_{0 \le \mb x, \mb z \le T} \bigg| \frac{n}{k} \tilde F_{\alpha, \beta} \left(\left(U_{(\lfloor kx_j\rfloor)}^j\right)_{j \in \alpha} , \left(U_{(\lfloor kz_j\rfloor)}^j\right)_{j \in \beta} \right)
\\&~~~~~~~~~~~~~~~~~~~~~~~~~~~~~~~~~~~~~- g_{\alpha, \beta}\left(\left(\frac{n}{k}U_{(\lfloor kx_j\rfloor)}^j\right)_{j \in \alpha} , \left(\frac{n}{k}U_{(\lfloor kz_j\rfloor)}^j\right)_{j \in \beta}\right) \bigg|\\
\Upsilon(n) &~=~ \max_{\alpha, \beta} \sup_{ 0 \le \mb x, \mb z \le T } \left| g_{\alpha, \beta}\left(\left(\frac{n}{k}U_{(\lfloor kx_j\rfloor)}^j\right)_{j \in \alpha} , \left(\frac{n}{k}U_{(\lfloor kz_j\rfloor)}^j\right)_{j \in \beta}\right) - g_{\alpha, \beta}(\mb x, \mb z) \right|.
\end{align*}
Now, considering (\ref{lem:eq-Wellner}) we have with probability greater than $1-\delta$ that
for every $1\le j \le d$, $U_{(\lfloor kT\rfloor )}^j ~\le~ 2T \frac{k}{n}~$, so that
\begin{align*} 
\Lambda(n) 
  ~\le~ 
\max_{\alpha, \beta}  \sup_{0 \le \mb x, \mb z \le 2T} \left| \frac{n}{k} \tilde F_{n, \alpha, \beta} \left( \frac{k}{n} \mb x , \frac{k}{n} \mb z \right) - \frac{n}{k} \tilde F_{\alpha, \beta} \left( \frac{k}{n} \mb x , \frac{k}{n} \mb z \right)  \right|.
\end{align*}
\noindent
Thus by Lemma \ref{lem:Fn-tildeF}, with probability at least $1-2\delta$,
\begin{align*}
 \Lambda(n) \le C d\sqrt{\frac{2 T}{k}\log\frac{1}{\delta}}.
\end{align*}
\noindent
Concerning $\Upsilon(n)$, we have the following decomposition:
\begin{align*}
 \Upsilon(n) &~\le~ \max_{\alpha, \beta} \sup_{ 0 \le \mb x, \mb z \le T } \bigg| g_{\alpha, \beta} \left( \frac{n}{k} \left(U_{(\lfloor kx_j\rfloor)}^j\right)_{j \in \alpha} , \frac{n}{k}\left(U_{(\lfloor kz_j\rfloor)}^j\right)_{j \in \beta}  \right) 
\\&~~~~~~~~~~~~~~~~~~~~~~~~~~~~~~~~~~~~~~~~~~~- g_{\alpha, \beta} \left(  \left(\frac{\lfloor kx_j\rfloor }{k}\right)_{j \in \alpha},\left(\frac{\lfloor kz_j\rfloor }{k}\right)_{j \in \beta} \right) \bigg| 
\\&~~~ ~+~     \max_{\alpha, \beta} \sup_{ 0 \le \mb x, \mb z \le T }
\left| g_{\alpha, \beta} \left(  \left(\frac{\lfloor kx_j\rfloor }{k}\right)_{j \in \alpha},\left(\frac{\lfloor kz_j\rfloor }{k}\right)_{j \in \beta} \right)
-g_{\alpha, \beta}(\mb x, \mb z) \right| 
\\&~=:~ \Upsilon_1(n) ~+~ \Upsilon_2(n)~.
\end{align*}
\noindent
The inequality in Lemma \ref{lem:g-alpha} allows us to bound the first term $\Upsilon_1(n)$:
\begin{align*}
\Upsilon_1(n) &~\le~ C  \max_{\alpha, \beta} \sup_{ 0 \le \mb x, \mb z \le T }~
\sum_{j \in \alpha} \left| \frac{\lfloor kx_j\rfloor }{k} - \frac{n}{k} U_{(\lfloor kx_j\rfloor )}^j \right| + \sum_{j \in \beta} \left| \frac{\lfloor kz_j\rfloor }{k} - \frac{n}{k} U_{(\lfloor kz_j\rfloor )}^j \right|\\
&~\le~  2 C \sup_{ 0 \le \mb x \le T }~
\sum_{1 \le j \le d} \left| \frac{\lfloor kx_j\rfloor }{k} - \frac{n}{k} U_{(\lfloor kx_j\rfloor )}^j \right|
\end{align*}
\noindent
so that by Lemma \ref{lem:U-x}, with probability greater than $1-(d+1)\delta$:
\begin{align*}
\Upsilon_1(n) ~\le~ Cd \sqrt{\frac{2 T}{k}\log{\frac{1}{\delta}}}~.
\end{align*}
\noindent
Similarly, $$\Upsilon_2(n) ~\le~  2C \sup_{0 \le \mb x \le T }~
\sum_{1\le j \le d} \left|\frac{\lfloor k x_j\rfloor }{k} - x_j\right| ~\le~C \frac{2d}{k} ~. $$ 
Finally we get, for every $n >0$, with probability at least $1- (d+3)\delta$, 
\begin{align*}
& \max_{\alpha, \beta}  \sup_{ 0 \le \mb x, \mb z \le T }
 \left| g_{n, \alpha, \beta}(\mb x, \mb z) - g_{\alpha, \beta}(\mb x, \mb z) \right| ~\le~ \Lambda(n) + \Upsilon_1(n) + \Upsilon_2(n) + \Xi(n)
\\ &~~~~~~\le~  Cd\sqrt{\frac{2T}{k}\log\frac{1}{\delta}} ~+~ \frac{2d}{k} 
~+~  \max_{\alpha, \beta} \sup_{ 0 \le \mb x, \mb z \le 2T}  \left|\frac{n}{k} \tilde F_{\alpha, \beta}(\frac{k}{n} \mb x , \frac{k}{n} \mb z)- g_{\alpha, \beta} (\mathbf{x}, \mathbf{z})\right|
\\ &~~~~~~\le~ C'd\sqrt{\frac{2T}{k}\log\frac{1}{\delta}} ~+~ 
  \max_{\alpha, \beta} \sup_{0 \le \mb x, \mb z \le 2T}
 \left|\frac{n}{k} \tilde F_{\alpha, \beta}(\frac{k}{n} \mb x , \frac{k}{n} \mb z)- g_{\alpha, \beta} (\mathbf{x}, \mathbf{z})\right|.
\end{align*}

\begin{remark}({\sc Bias term})
\label{rk:bias}
It is classical (see \cite{Qi97} p.174 for details) to extend the simple convergence (\ref{g-alpha}) to the uniform version on $[0,T]^d$. It suffices to subdivide $[0,T]^d$ and to use the monotonicity in each dimension coordinate of $g_{\alpha, \beta}$ and $\tilde F_{\alpha, \beta}$.
Thus,  
$$\sup_{0 \le \mb x, \mb z \le 2T}\left|\frac{n}{k} \tilde F_{\alpha, \beta}( \frac{k}{n} \mb x, \frac{k}{n} \mb z)- g_{\alpha, \beta}(\mb x, \mb z)\right| \to 0$$ for every $\alpha$ and $\beta$. Note also that by taking a maximum on a finite class we have the convergence of the maximum uniform bias to $0$:
\begin{align}
\label{unif_conv}
\max_{\alpha, \beta}~ \sup_{0 \le \mb x, \mb z \le 2T}\left|\frac{n}{k} \tilde F_{\alpha, \beta}( \frac{k}{n} \mb x, \frac{k}{n} \mb z)-g_{\alpha, \beta} (\mb x, \mb z)\right| \to 0.
\end{align}
\noindent
\end{remark}

\subsection{Proof of Lemma \ref{lemma_simplex}}
First note that as the $\Omega_\beta$'s
form a partition of the simplex $S_\infty^{d-1}$ and that $\Omega_\alpha^{\epsilon,\epsilon'} \cap \Omega_\beta = \emptyset $ as soon as $\alpha \not \subset \beta$, we have 
$$\Omega_{\alpha}^{\epsilon, \epsilon'} ~=~ \bigsqcup_\beta
\Omega_\alpha^{\epsilon, \epsilon'} \cap \Omega_\beta ~=~ \bigsqcup_{\beta \supset
  \alpha} \Omega_\alpha^{\epsilon, \epsilon'} \cap \Omega_\beta .$$ 

Let us recall that as stated in Lemma~\ref{lem:continuousPhi}), $ \Phi$ is concentrated on the (disjoint) edges 
\begin{align*}
   \Omega_{\alpha,i_0} = \{\mb x: \; \ninf{\mb x}  = 1,\; x_{i_0} = 1,~~& 0<  x_i < 1 ~~\text{~for~} i \in \alpha \setminus \{i_0\}\\
&x_i=0 ~~~~\text{~~~ for } i\notin \alpha ~~~~~~~\}
\end{align*}
and that the restriction $\Phi_{\alpha,i_0}$ of $\Phi$ to $\Omega_{\alpha,i_0}$ is absolutely continuous \wrt~the Lebesgue measure $\ud x_{\alpha\setminus{i_0}}$ on the cube's edges, whenever $|\alpha|\ge 2 $.
\noindent
By (\ref{eq:decomposePhi}) we have, for every $\beta \supset \alpha$, 
\begin{align*}
\Phi(\Omega_\alpha^{\epsilon, \epsilon'} \cap \Omega_{\beta})~~&=~~ \sum_{i_0 \in \beta} ~\int_{\Omega_\alpha^{\epsilon, \epsilon'} \cap \Omega_{\beta,i_0}}  ~\frac{\ud \Phi_{\beta,i_0}}{\ud x_{ \beta \setminus i_0}}(x) ~\ud x_{\beta \setminus i_0}\\
\Phi(\Omega_\alpha)~~&=~~ \sum_{i_0 \in \alpha} ~\int_{\Omega_{\alpha,i_0}}  ~\frac{\ud \Phi_{\alpha,i_0}}{\ud x_{ \alpha \setminus i_0}}(x) ~\ud x_{\alpha \setminus i_0}~.
\end{align*}
Thus,
\begin{align*}
\Phi(\Omega_\alpha^{\epsilon, \epsilon'}) - \Phi(\Omega_\alpha) &~=~ \sum_{\beta \supset \alpha} \sum_{i_0 \in \beta} ~\int_{\Omega_\alpha^{\epsilon, \epsilon'} \cap \Omega_{\beta,i_0}}  ~\frac{\ud \Phi_{\beta,i_0}}{\ud x_{ \beta \setminus i_0}}(x) ~\ud x_{\beta \setminus i_0}\\
&~~~~~~~~~~~~~~~~~~~~~~~~~~~~~~~~~~~~-~\sum_{i_0 \in \alpha} ~\int_{\Omega_{\alpha,i_0}}  ~\frac{\ud \Phi_{\alpha,i_0}}{\ud x_{ \alpha \setminus i_0}}(x) ~\ud x_{\alpha \setminus i_0}\\
&~=~ \sum_{\beta \supsetneq \alpha} \sum_{i_0 \in \beta} ~\int_{\Omega_\alpha^{\epsilon, \epsilon'} \cap \Omega_{\beta,i_0}}  ~\frac{\ud \Phi_{\beta,i_0}}{\ud x_{ \beta \setminus i_0}}(x) ~\ud x_{\beta \setminus i_0}\\
&~~~~~~~~~~~~~~~~~~~~~-~\sum_{i_0 \in \alpha} ~\int_{\Omega_{\alpha,i_0} \setminus (\Omega_\alpha^{\epsilon, \epsilon'} \cap \Omega_{\alpha,i_0})}  ~\frac{\ud \Phi_{\alpha,i_0}}{\ud x_{ \alpha \setminus i_0}}(x) ~\ud x_{\alpha \setminus i_0},
\end{align*}
so that by eq\ref{eq:supDensity}, 
\begin{align}
\label{pr:decompPhi}
|\Phi(\Omega_\alpha^{\epsilon, \epsilon'}) - \Phi(\Omega_\alpha)| &~\le~ \sum_{\beta \supsetneq \alpha} M_\beta \sum_{i_0 \in \beta} ~\int_{\Omega_\alpha^{\epsilon, \epsilon'} \cap \Omega_{\beta,i_0}} ~\ud x_{\beta \setminus i_0}\\
\nonumber&~~~~~~~~~~~~~~~~~~~~~~~~~~~+~ M_\alpha \sum_{i_0 \in \alpha} ~\int_{\Omega_{\alpha,i_0} \setminus (\Omega_\alpha^{\epsilon, \epsilon'} \cap \Omega_{\alpha,i_0})}  ~\ud x_{\alpha \setminus i_0}~.
\end{align}
Without loss of generality we may assume that $\alpha =\{1,...,K\}$ with $K \le d$. 
Then, for $\beta \supsetneq \alpha$, $\int_{\Omega_\alpha^{\epsilon, \epsilon'} \cap \Omega_{\beta,i_0}} ~\ud x_{\beta \setminus i_0}$ is smaller than $(\epsilon')^{|\beta| - |\alpha|}$ and is null as soon as $i_0 \in \beta \setminus \alpha$. To see this, assume for instance that $\beta = \{1,...,P\}$ with $P>K$. Then 
\begin{align*}
\Omega_\alpha^{\epsilon, \epsilon'} ~\cap~ \Omega_{\beta,i_0} = \{ \epsilon < x_1,...,x_K \le 1,~&x_{K+1},...,x_P \le \epsilon' ,~x_{i_0}=1,\\ &x_{P+1}=...=x_d=0~~~~~~~~~\}
\end{align*}
which is empty if $i_0 \ge K+1$ (\ie~$i_0 \in \beta \setminus \alpha$) and which fulfills if $i_0 \le K$ $$\int_{\Omega_\alpha^{\epsilon, \epsilon'} \cap \Omega_{\beta,i_0}} ~\ud x_{\beta \setminus i_0} \le (\epsilon')^{P-K}.$$
The first term in (\ref{pr:decompPhi}) is then bounded by $\sum_{\beta \supsetneq \alpha } M_\beta |\alpha| (\epsilon')^{|\beta|-|\alpha|}$.
Now, concerning the second term in (\ref{pr:decompPhi}), $\Omega_{\alpha}^{\epsilon, \epsilon'} \cap \Omega_{\alpha,i_0} ~=~\{\epsilon < x_1,...,x_K \le 1, x_{i_0}=1,~x_{K+1},...,x_d=0 \} $ and then
\begin{align*}
\Omega_{\alpha,i_0} \setminus (\Omega_\alpha^{\epsilon, \epsilon'} \cap \Omega_{\alpha,i_0}) = \bigcup_{l=1,...,K} \Omega_{\alpha,i_0} \cap \{ x_l \le \epsilon \},
\end{align*}
so that $\int_{\Omega_{\alpha,i_0} \setminus (\Omega_\alpha^{\epsilon, \epsilon'} \cap \Omega_{\alpha,i_0})}  ~\ud x_{\alpha \setminus i_0}~ \le~ K\epsilon = |\alpha| \epsilon$.
The second term in (\ref{pr:decompPhi}) is thus bounded by $M |\alpha|^2 \epsilon$.
Finally, 
(\ref{pr:decompPhi}) implies 
\begin{align*}
|\Phi(\Omega_\alpha^{\epsilon, \epsilon'}) - \Phi(\Omega_\alpha)| \le |\alpha|  \sum_{\beta \supsetneq \alpha} M_\beta (\epsilon')^{|\beta|-|\alpha|} + M |\alpha|^2 \epsilon.
\end{align*}
To conclude, observe that by Assumption \ref{hypo:abs_continuousPhi},
$$\sum_{\beta \supsetneq \alpha} M_\beta (\epsilon')^{|\beta|-|\alpha|} \le \sum_{\beta \supsetneq \alpha} M_\beta (\epsilon') \le \epsilon' \sum_{|\beta| \ge 2} M_\beta \le \epsilon' M
$$
The result is thus proved.

\subsection{Proof of Remark~\ref{rk:optim}}
Let us prove that $Z_n$, conditionally to the event $\{\|\frac{k}{n} V_1\|_\infty \ge 1\}$, converges in law.
Recall that $Z_n$ is a $(2^d-1)$-vector defined by $Z_n(\alpha)=\mathds{1}_{\frac{k}{n} \mb V_1 \in R_\alpha^\epsilon}$ for all $\alpha \subset \{1,\ldots,d\}, \alpha \neq \emptyset$.
Let us denote $1_\alpha = (\mathds{1}_{j=\alpha})_{j=1,\ldots, 2^d-1}$ where we implicitely define the bijection between $\mathcal{P}(\{1,\ldots,d\}) \setminus \emptyset$ and $\{1,\ldots,2^d-1\}$.
Since the $R_\alpha^\epsilon$'s, $\alpha$ varying, form a partition of $[\mb 0, \mb 1]^c$,
$\mathbb{P}(\exists \alpha, Z_n = 1_\alpha ~|~\|\frac{k}{n} \mb V_1 \|_\infty \ge 1)=1$ and 
$Z_n = 1_\alpha \Leftrightarrow Z_n(\alpha) = 1 \Leftrightarrow \frac{k}{n} \mb V_1 \in R_\alpha^\epsilon$, so that
$$\mathbb{E}\left[\Phi(Z_n) \mathds{1}_{\|\frac{k}{n} \mb V_1 \|_\infty \ge 1}\right] = \sum_\alpha \Phi(1_\alpha) \mathbb{P}(Z_n(\alpha) = 1).$$
\noindent
Let $\Phi: \mathbb{R}^{2^d-1} \to \mathbb{R}_+$ be a measurable function. Then
\begin{align*}
\mathbb{E}\left[\Phi(Z_n) ~|~ \|\frac{k}{n} \mb V_1 \|_\infty \ge 1\right] ~=~ \mathbb{P}\left[\|\frac{k}{n} \mb V_1 \|_\infty \ge 1\right]^{-1}~~ \mathbb{E}\left[\Phi(Z_n) \mathds{1}_{\|\frac{k}{n} \mb V_1 \|_\infty \ge 1}\right].
\end{align*}
Now, $\mathbb{P}\left[\|\frac{k}{n} \mb V_1 \|_\infty \ge 1\right] = \frac{k}{n} \pi_n$ with $\pi_n \to \mu([\mb 0, \mb 1]^c)$, so that
$$\mathbb{E}\left[\Phi(Z_n) ~|~ \|\frac{k}{n} \mb V_1 \|_\infty \ge 1\right] = \pi_n^{-1} \frac{n}{k}\left(\sum_\alpha \Phi(1_\alpha) \mathbb{P}(Z_n(\alpha) = 1) \right).$$
Using $\frac{n}{k} \mathbb{P}\left[Z_n(\alpha)=1\right] = \frac{n}{k}\mathbb{P}\left[\frac{k}{n} \mb V_1 \in R_\alpha^\epsilon\right] \to \mu(R_\alpha^\epsilon)$, we find that
$$\mathbb{E}\left[\Phi(Z_n) ~|~ \|\frac{k}{n} \mb V_1 \|_\infty \ge 1\right] \to \sum_\alpha \Phi(1_\alpha) \frac{\mu(R_\alpha^\epsilon)}{\mu([\mb 0, \mb 1]^c)},$$
which achieves the proof.

\section{Experiments curves}
\label{appendix_exp}
\begin{figure}[H]
  \centering
  \includegraphics[width = \textwidth]{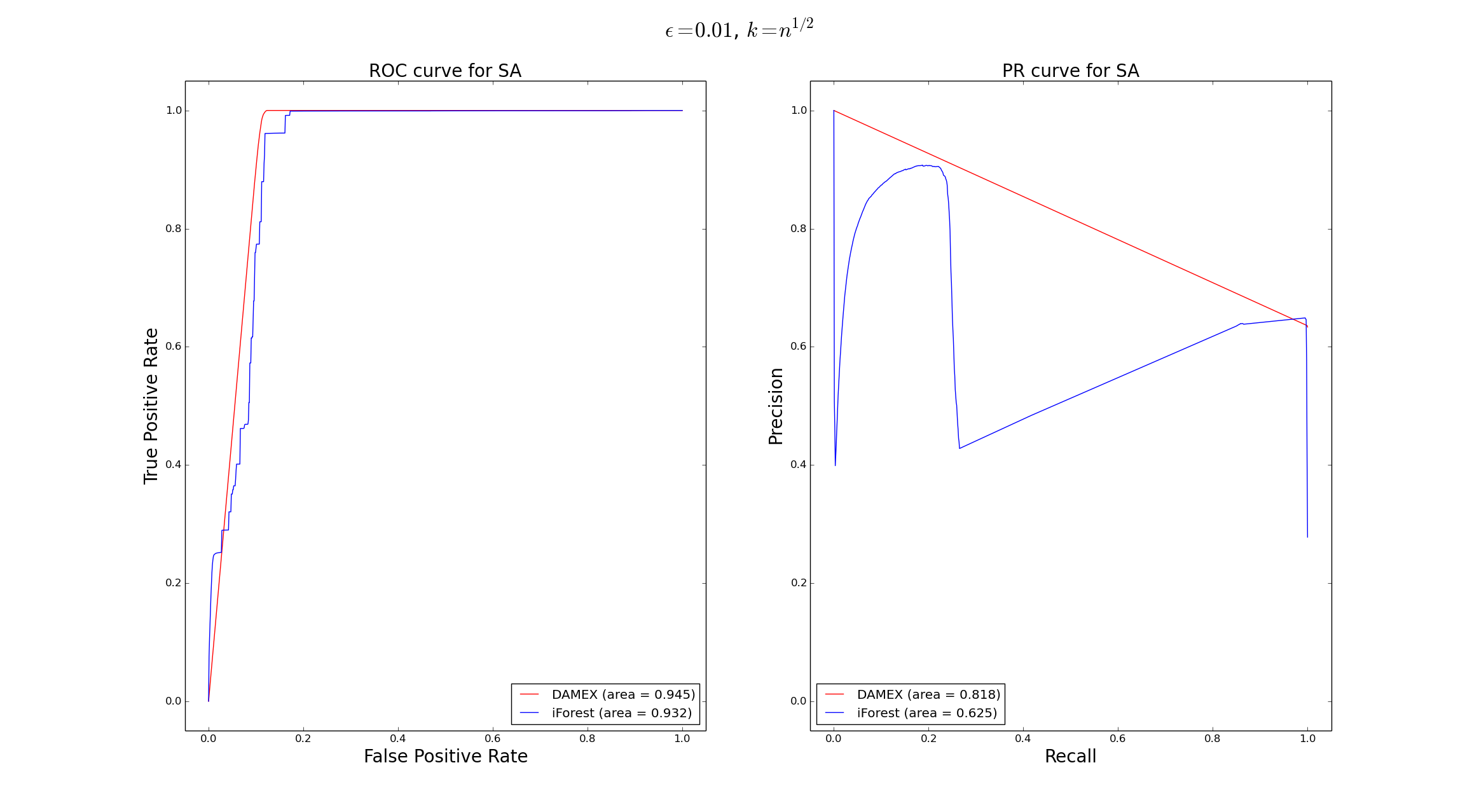}
  \caption{SA dataset, default parameters}
\end{figure}

\begin{figure}[H]
  \centering
  \includegraphics[width = \textwidth]{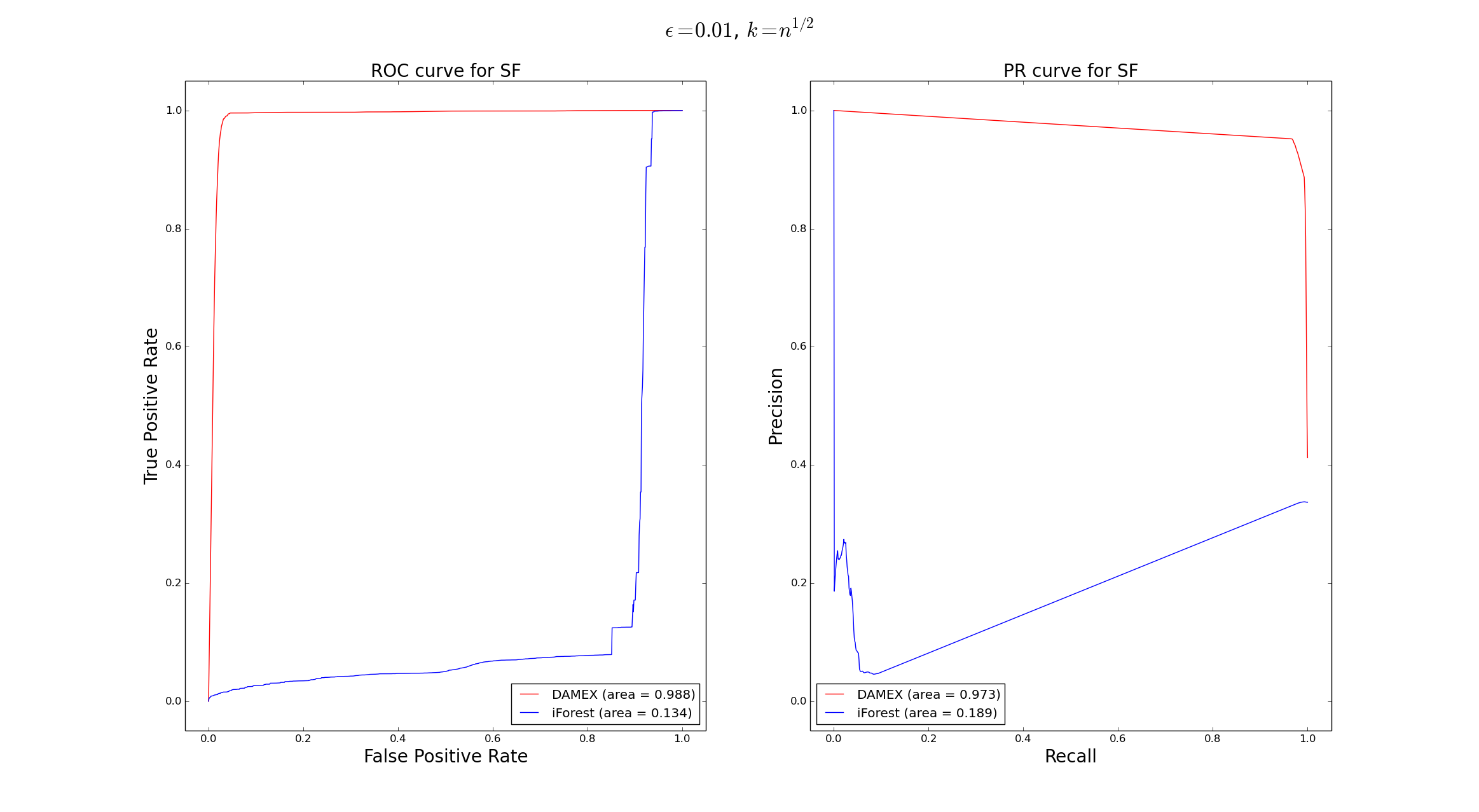}
  \caption{shuttle dataset, default parameters}
\end{figure}

\begin{figure}[H]
  \centering
  \includegraphics[width = \textwidth]{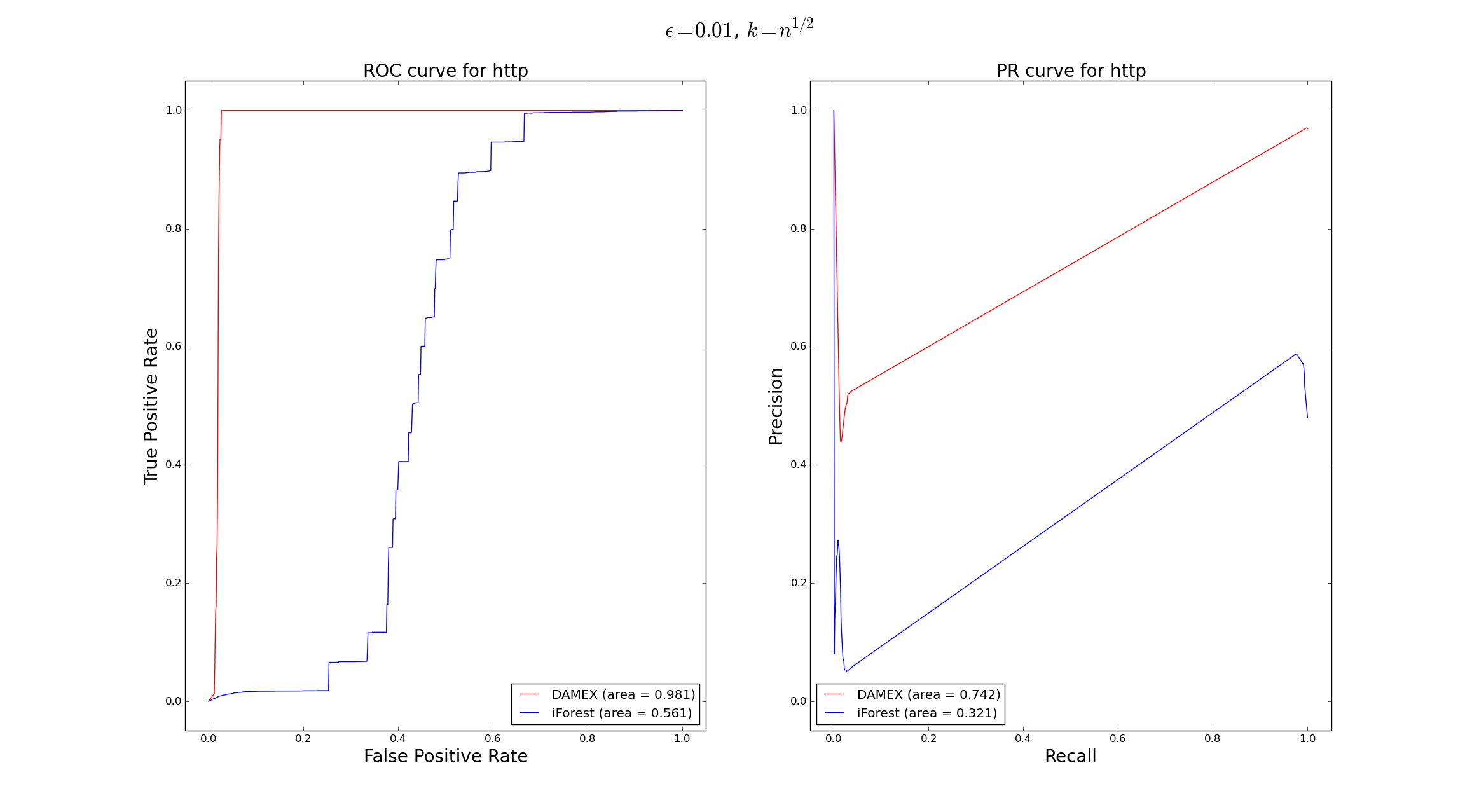}
  \caption{http dataset, default parameters}
\end{figure}

\section*{References}
\bibliographystyle{elsarticle-harv} 
\bibliography{mvextrem.bib}







\end{document}